\documentclass[11pt]{article}
\usepackage{amsmath,amssymb, enumerate, amsthm}
\usepackage{fullpage}
\usepackage{graphicx,color,xcolor}
\usepackage[export]{adjustbox}
\usepackage{subfig}
\usepackage{amsfonts}
\usepackage{tabularx}
\usepackage{multirow}
\usepackage[title]{appendix}
\usepackage{mathtools, commath}
\usepackage{soul}
\usepackage{hyperref}
\usepackage{tikz}
\usepackage[linesnumbered,lined,boxed,ruled,algo2e]{algorithm2e}
\usepackage{epstopdf}
\AppendGraphicsExtensions{.tiff}

\newtheorem{theorem}{Theorem}

\newtheorem{lemma}{Lemma}

\newtheorem{definition}{Definition}

\DeclareMathOperator*{\argmin}{argmin}

\renewcommand{\Tilde}{\widetilde}

\newcommand{\Cov}{\mathrm{Cov}}

\begin{document}

\title{Texture Edge detection by Patch consensus (TEP)}

\author{Guangyu Cui \thanks{School of Mathematics, Georgia Institute of Technology, Atlanta, GA, USA (gcui8@gatech.edu)} ~and Sung Ha Kang\thanks{School of Mathematics, Georgia Institute of Technology, Atlanta, GA, USA (kang@math.gatech.edu)}}

\date{}
\maketitle

\begin{abstract} 
We propose Texture Edge detection using Patch consensus (TEP) which is a training-free method to detect the boundary of texture.  We propose a new simple way to identify the texture edge location, using the consensus of segmented local patch information.  While on the boundary, even using local patch information, the distinction between textures are typically not clear, but using neighbor consensus give a clear idea of the boundary.
We utilize local patch, and its response against neighboring regions, to emphasize the similarities and the differences across different textures. The step of segmentation of response further emphasizes the edge location, and the neighborhood voting gives consensus and stabilize the edge detection.
We analyze texture as a stationary process to give insight into the patch width parameter verses the quality of edge detection.  We derive the necessary condition for textures to be distinguished, and analyze the patch width with respect to the scale of textures.  
Various experiments are presented to validate the proposed model. 
\end{abstract}



\section{Introduction}\label{sec:intro}

Texture has been explored for decades \cite{van1985texture} and fruitful results are established for different type of textures: 
Markov random field \cite{raad2018survey} is widely used for texture synthesis. Its generative feature fits well with the randomness of certain types of textures (wood surface, sand); Lattice based method \cite{liu2004computational} is powerful in modeling highly symmetric and periodic texture, when the texton is relatively well defined (wall paper, honeycomb); Frequency/Wavelet analysis \cite{livens1997wavelets, unser1995texture} utilize spacial filters to vectorize textures and plays an important role in image compression \cite{lewis1992image} and texture classification and segmentation \cite{wang1990texture, jain1991unsupervised}. We refer to \cite{tuceryan1993texture,raad2018survey} for a comprehensive review of classical models. 
Textures remains to be a challenging topic, since there is no general nor precise definition of textures, and the boundaries of different textures are especially difficult to recognize.  

We explore texture edge detection and segmentation. 
The classical Canny edge detection \cite{canny1986computational} detects edge locations by thinning the mask function, which is computed by applying thresholds to the magnitude of $|\nabla U_0|$. 
One of the most well-known variational segmentation models is the Mumford-Shah functional \cite{mumford1989optimal}, 
\begin{align}\label{eq:mumford-shah}
	E_{MS}(U, \Gamma) = \alpha\int_{\Omega\backslash\Gamma}|\nabla U|^2 dx + \beta 
 \mathcal{H}^1(\Gamma) 
 + \int_{\Omega\backslash\Gamma}(U-U_0)^2 dx.
\end{align}
Here, $\alpha$ and $\beta$ are positive parameters, $\Omega \subset \mathbb{R}$ is a bounded image domain, $U_0:\Omega\to\mathbb{R}$ is the given image, and $ \mathcal{H}^1(\Gamma)$ denotes one-dimensional Hausdorff measure of the object boundary $\Gamma$.  Chan and Vese \cite{active} proposed using level set method, and it gives very effective piece-wise constant segmentation.
For textured image segmentation, some texture descriptors, such as Gabor filter, can be used with these models.   
Gabor filter  \cite{mehrotra1992gabor} can detect localized frequency response in varied orientations and scales. 
Figure \ref{fig:intro} illustrates challenges of texture edge detection.
For a comprehensive review of classical texture segmentation models, we refer to \cite{ilea2011image}.
In real images, there are many different types of textures and it is typically difficult to find a proper filter bank that is suitable for all types of images. 
More recent network based approaches with  data-adaptive property is capable of accomplishing high level tasks as semantic segmentation, e.g., \cite{bertasius2015deepedge, 9093290, xie15hed, he2019bi, liu2017richer}, assuming the network is well-trained. 
\begin{figure}
    \centering
        \begin{tabular}{ccc}
(a) & (b) & (c) \\   
\includegraphics[width=0.25\textwidth]{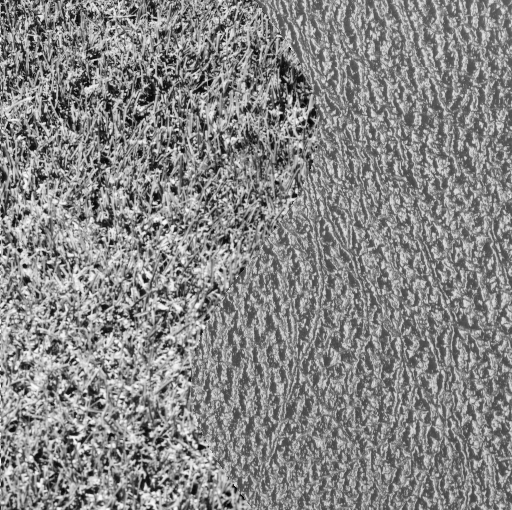}&
    \includegraphics[width=0.25\textwidth]{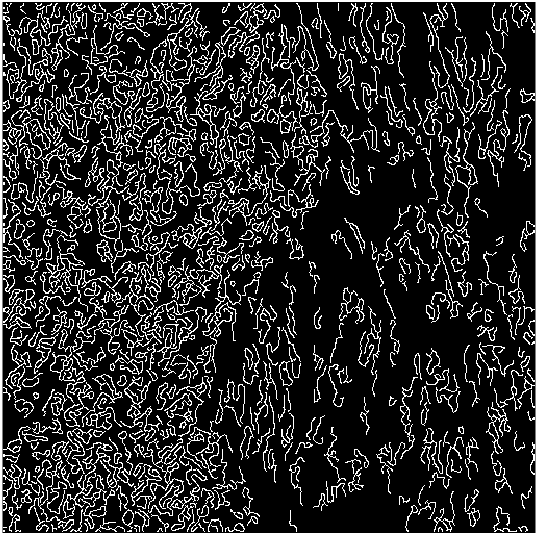}&
     \includegraphics[width=0.25\textwidth]{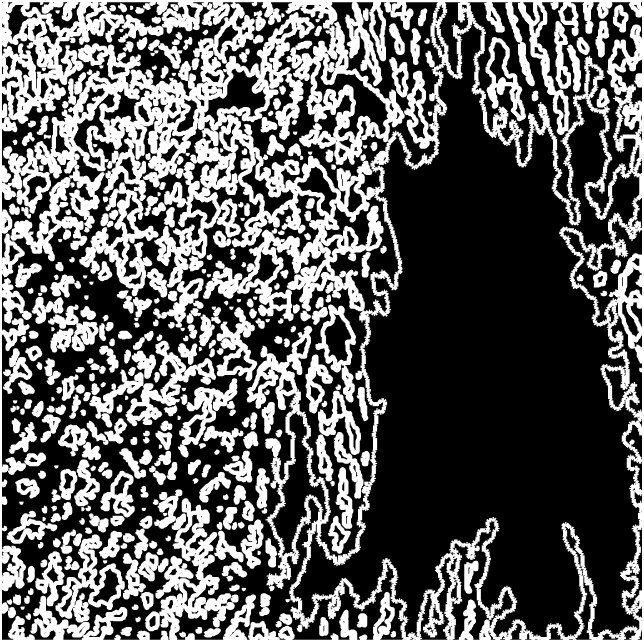}
\end{tabular}
\caption{Challenges of texture edge detection.  (a) An given image with textures.  (b) Canny edge detection. (c) Chan-Vese segmentation. \label{fig:intro}}
\end{figure}

In this paper, we propose filter-free and training-free approach utilizing local patches response to capture the similarities and the differences between textures.  Non-local filter \cite{nonlocal} for image denoising averages pixel intensity values over similar image patches effectively. 
In \cite{efros1999texture}, the authors extend this to texture synthesis algorithm, where the pixels are selected over similar image patches to regenerate textures.  
One of the difficulties in applying such nonlocal filter for texture edge detection is that near the boundary of texture, there may not be similar patches to give the clear boundary, different from denoising non textures edges.  

We use local patch information utilizing similar responses within one texture, and different responses against different textures, and propose a Texture Edge detection method by Patch consensus (TEP). 
We propose a consensus based edge detection, which utilize the fact that away from the boundary often gives a clear idea about where the boundary of the texture should be located.  We analyze the statistical condition for the texture edge to be detected by patch-wise similarity, and explore the relation between the patch width and the performance of the proposed method.  
The contributions of this paper are as follows:
\begin{itemize}
\item{We propose a simple training-free filter-free Texture Edge detection method using patch consensus (TEP). }
\item{We statistically analyze when the texture can be separated by the patch consensus. }
\item{Numerical results are presented to validate the proposed model. }
\end{itemize}

The paper is outlined as follows: In Section \ref{sec:method}, the details of texture edge detection with Patch consensus (TEP) is presented.  Statistical analysis of the proposed model is provided in Section \ref{sec:analysis}.  In Section \ref{sec:Numerical}, we present the algorithms and numerical implementation details, and  in Section \ref{sec:experiments} various experiments with comparisons and applications are presented.

\section{The proposed model: Texture Edge detection by  Patch consensus (TEP)}\label{sec:method}

Let the discrete image domain be $\Omega =  [1,2,\dots, M]\bigoplus[1,2,\dots,N]$, and let the matrix $U \in \mathbb{R}^{M \times N} $ denote the given image, where $U[\mathbf{x}]= U[x_1, x_2]$ represents the intensity value at a pixel location $\mathbf{x} \in \Omega$. We consider a square neighborhood of $\mathbf{x}$ with the width  2$r+1 \in \mathbb{Z}^+$ to be
\begin{equation*}
\mathcal{B}_r(\mathbf{x}) = \{\mathbf{y}\in\Omega \mid \|\mathbf{y}-\mathbf{x}\|_\infty \leq r\}, 
\end{equation*}
and we denote the vector version of the image patch of $\mathcal{B}_r(\mathbf{x})$ to be $\vec{\mathcal{P}}(\mathbf{x}) \in \mathbb{R}^{d}$ that $d = (2r+1)^2$. We refer to $r$ as the patch width parameter.   
We set the order of the entry in the vector to be column-wise from left to right, 
i.e., let the matrix $C$ be the $\sqrt{d}\times\sqrt{d}$ image, the transformation matrix pair $A\in\mathbb{R}^{{d}\times \sqrt{d}}, B\in\mathbb{R}^{\sqrt{d}\times 1}$, where
\begin{align*}
    A = \left(\begin{array}{c}
         1  \\
         1  \\
         \vdots \\
         1
    \end{array}\right)\bigotimes I, \qquad
    B = \left(\begin{array}{c}
         1  \\
         0  \\
         \vdots \\
         0 
    \end{array}\right), \qquad \text{then} \quad
    \vec{\mathcal{P}}(\mathbf{x}) = ACB
\end{align*}
which transforms a square patch in $\mathbb{R}^{\sqrt{d}\times \sqrt{d}}$ to a vector in $\mathbb{R}^{d}$.
Here $\bigotimes$ denotes Konecker product, and $I$ is $\sqrt{d}$ dimensional identity matrix. 

The main idea of the proposed method, Texture Edge detection by Path consensus (TEP), is as follows.  From a \textit{local patch} $\vec{\mathcal{P}}(\mathbf{x})$, a \textit{patch response} $\mathcal{R}(\mathbf{y};\mathbf{x})$ is considered.  We segment these patch responses $\mathcal{R}(\mathbf{y};\mathbf{x})$ to emphasize the similarities and the differences of patch responses.  Then, we collect these segmentation boundaries and construct the edge function $V$ in  $\Omega$. 

\textbf{[Step 1]}
For each patch $\vec{\mathcal{P}}(\mathbf{x})$, we define the patch response in a larger domain as 
\begin{equation} \label{e:response}
\mathcal{R}(\mathbf{y};\mathbf{x}) = \frac{1}{(2r+1)^2}\|\vec{\mathcal{P}}(\mathbf{y}) - \vec{\mathcal{P}}(\mathbf{x})\|_2^2 \geq 0.
\end{equation}
Here $\mathbf{y} \in \mathcal{B}_R(\mathbf{x})$ with  $R > r$ represents the half width of the comparison neighborhood.  We refer to $R$ as the large comparison region width parameter.  
This $\mathcal{R}(\mathbf{y};\mathbf{x})$ measures the similarity of a patch at  $\mathbf{x}$ and a patch at $\mathbf{y}$.  When the  patches $\vec{\mathcal{P}}(\mathbf{x})$ and $\vec{\mathcal{P}}(\mathbf{y})$ are similar, it gives near zero value, and when they are very different, it gives a high value.  For computational efficiency, we take $\mathbf{y}$ from the neighborhood $\mathcal{B}_R(\mathbf{x})$, but one may use $\mathbf{y} \in \Omega$.  

\textbf{[Step 2]} To emphasize the texture differences and capture the edge information more clearly, we segment the response $\mathcal{R}(\mathbf{y};\mathbf{x})$ on $\mathcal{B}_R(\mathbf{x})$ using the following unsupervised multiphase segmentation model \cite{unsupervised}:
\begin{equation}\label{eq:multiphase}
    E_{\text{seg}}(\chi_i, c_i, K| \mathcal{R}) 
    = \lambda  \left(\sum_{i=1}^K \frac{P_i}{A_i} \right)  \mathcal{H}^1(\Gamma) 
    + \sum_{i=1}^K\int_{\chi_i} | \mathcal{R}(\mathbf{y};\mathbf{x}) - c_i|^2 d\mathbf{x}
\end{equation}
where $\chi_i$ is the indicator function of each phase $i$ which partitions $B_R(\mathbf{x}) = \bigcup_{i=1}^{K} \chi_i$,  $K$ is the number of phases, $\mathcal{H}^1$ denotes one-dimensional Hausdorff measure, $\Gamma = \cup_{i=1}^K \{\partial \chi_i\}$ is the set of all boundaries, and $c_i = \int_{\chi_i} \mathcal{R}(\mathbf{y};\mathbf{x}) \; d\mathbf{x} /\int_{\chi_i}d\mathbf{x}$ is the intensity average of the phase $i$.  Here the scale term $P_i/A_i = \mathcal{H}^1(\partial \chi_i) /\int_{\chi_i}  d\mathbf{x} $ is the perimeter over the area of each phase $i$.  
This model (\ref{eq:multiphase}) automatically finds the number of phases $K$ by a greedy algorithm.   In this paper, we bound the number of phases to be $K\in\{1, 2\}$, thus it finds either one or two phases within the response $\mathcal{R}(\mathbf{y};\mathbf{x})$. 
We define the local edge function  to be  
\begin{equation} \label{eq:w}
    \mathcal{W}(\mathbf{y};\mathbf{x}) = \frac{1}{2} \sum_{i=1}^K |\nabla \chi_i|.
\end{equation}
This represents the edge from the point of view of patch $\vec{\mathcal{P}}(\mathbf{x})$. 

\textbf{[Step 3]} A local response for points on the boundary of texture doesn't give a good edge information in general, thus we use consensus and collect the segmented patches to determine the edge function $V(\mathbf{x})$, by superposing $\mathcal{W}(\mathbf{y};\mathbf{x})$ for $\forall \mathbf{x}\in\Omega$;
\begin{align} \label{eq:V}
	V(\mathbf{x}) = \frac{1}{\abs{\mathcal{B}_R(\mathbf{x})}} \sum_{\mathbf{y}\in \mathcal{B}_R(\mathbf{x})} \mathcal{W}(\mathbf{x};\mathbf{y}).
\end{align}
This becomes a non-binary edge function $V(\mathbf{x}):\Omega\to[0, 1]$ representing the ratio of $\mathbf{x}$'s neighbors $\mathbf{y}\in \mathcal{B}_R(\mathbf{x})$ that voted $\mathbf{x}$ as an edge pixel.  This superposition gives consensus among patch responses.  Even when the texture boundary is not very clear, points away from the boundary can still give a good information about the edge location. 
\begin{figure}
	\begin{center}
	\includegraphics[width =\textwidth]{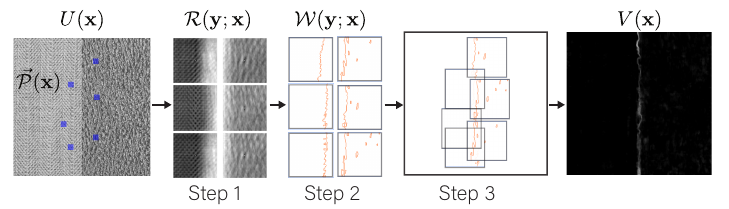}
	\end{center}
	\caption{Texture Edge detection by Patch consensus (TEP). For each patch $\vec{\mathcal{P}}(\mathbf{x})$ in the given image $U$, the patch response $\mathcal{R}(\mathbf{y};\mathbf{x})$ is computed.  The patch response is segmented, and the boundary of the phases gives the local edge $\mathcal{W}(\mathbf{y},\mathbf{x})$.  The edge function $V(\mathbf{x})$ is computed by the consensus of the local edge function $\mathcal{W}(\mathbf{y},\mathbf{x})$.}
	\label{fig:flow}
\end{figure}

The flowchart of the proposed model TEP is presented in Figure \ref{fig:flow}: for each pixel $\mathbf{x} \in \Omega$ and its patch $\vec{\mathcal{P}}(\mathbf{x})$, the patch responses on a larger domain $B_R (\mathbf{x})$ is computed as $\mathcal{R}(\mathbf{y};\mathbf{x})$.  We use  unsupervised segmentation to segment the patch responses to emphasize the similarities and the differences among these patch responses. 
The gradient of phases is used to compute the local edges $\mathcal{W}(\mathbf{y};\mathbf{x})$.  Finally, the consensus is used to get the edge map $V(\mathbf{x})$. 
Since we use the observer patch $\vec{\mathcal{P}}(\mathbf{x})$ as input, the proposed method is self-adaptive to the image without the need of training.  This also reduces the number of hyper-parameters needed in filter based approaches.  The parameters needed for TEP are the patch width parameter $r$ of $\vec{\mathcal{P}}(\mathbf{x})$, the large comparison region width parameter $R$, and one regularity parameter $\lambda$ for the unsupervised multi-phase segmentation.

\section{Analytical properties of the proposed model}  \label{sec:analysis}

In this section, we statistically analyze when the texture can be separated by the patch consensus. In particular, we model the texture as random fields, derive the necessary conditions for our model to generate distinguishable patch responses for different textures in the sense of patch-wise Euclidean distance, and study the roles of the patch width parameter $r$, and the large comparison region width parameter $R$.

\subsection{Texture as Stationary Random Field} 

Random field models the self-similarity property of the natural stochastic textures well that statistical approaches are proposed for structure-texture decomposition and image denoising \cite{khawaled2019self, zachevsky2016statistics, xu2020structure}. 
In this paper, we model texture as a two dimensional Gaussian random field \cite{adler2009random} defined on pixels $\Omega$, and study how the decay of correlation of the texture random field helps to identify texture boundaries from the patch responses for stochastic textures. In the context of discussing image patches as random vectors, we use calligraphic letter $\vec{\mathcal{P}}$ to denote random vector and lowercase letter $\vec{v}$ to denote a concrete vector of the same size as $\vec{\mathcal{P}}$. We start with introducing the definitions of Gaussian random field and its related properties. 

\begin{definition}\label{def:Gaussian random field}
    Let $\mathbf{x}\in\Omega\subset\mathbb{Z}^2$ be the pixel index. The set of random variables $\mathcal{P} = \{\mathcal{P}(\mathbf{x})\}_{\mathbf{x}\in\Omega}$ is a \textbf{Gaussian random field}, if $\vec{\mathcal{P}}(\mathbf{x}) = [\mathcal{P}(\mathbf{x}_1), \mathcal{P}(\mathbf{x}_2), \dots, \mathcal{P}(\mathbf{x}_d)]^T$ is a $d$-dimensional Gaussian random vector for arbitrary choices of indices $\mathbf{x}_1, \mathbf{x}_2, \dots, \mathbf{x}_d \in\Omega$, where $d\in\mathbb{Z}^+$ and $\mathcal{P}(\mathbf{x}_i)$ denotes a Gaussian variable indexed by pixel location $\mathbf{x}_i$.    The probability density of $\vec{\mathcal{P}}(\mathbf{x}) = \vec{v}$ is given by
    \begin{align}
        \phi(\vec{v}) = \frac{1}{(2\pi)^{d/2}|\Sigma|^{1/2}}e^{-\frac{1}{2}(\vec{v}-\vec{\mu}_p)^T\Sigma_p^{-1}(\vec{v}-\vec{\mu}_p)},\notag
    \end{align}
    where $\vec{\mu}_p = \mathbb{E}(\vec{\mathcal{P}}(\mathbf{x}))$ is the expectation vector and $\Sigma_p = \Cov(\vec{\mathcal{P}}(\mathbf{x}))$ is the nonnegative definite $d\times d$ covariance matrix.
\end{definition}

A Gaussian random field is completely determined by its first and the second moments, i.e., its mean $\vec{\mu}$ and covariance $\Sigma$, and  Gaussian distribution is suitable for many natural stochastic textures \cite{zachevsky2016statistics}. 
In this paper, we assume fast decaying of the correlation with respect to pixelwise distance $\|\mathbf{x}-\mathbf{y}\|_2$ and choose a squared exponential covariance function such as 
\begin{align}\label{eq:exp square covariance}
    \Cov(\mathcal{P}(\mathbf{x}_1), \mathcal{P}(\mathbf{x}_2)) = \gamma_p(\mathbf{x}_1, \mathbf{x}_2) = \sigma_p^2\exp\left(-\frac{\|\mathbf{x}_1 - \mathbf{x}_2\|_2^2}{2l_p^2}\right),
\end{align}
which makes the random field $\mathcal{P}$ stationary and isotropic, here $\sigma_p > 0$ is the magnitude parameter and $l_p>0$ is the decaying rate parameter. We remark that we choose the squared exponential decaying covariance \eqref{eq:exp square covariance} for the convenience of computation, and the derivations of this section can be generalized to decaying covariance functions of any order.  
For textures with spatially repetitive patterns, it is natural to assume the corresponding random field to be stationary \cite{zachevsky2016statistics}. 
\begin{definition}\label{def:stationary}
    A Gaussian random field $\mathcal{P}$ is called \textbf{stationary}, if for every $\mathbf{x}_1, \mathbf{x}_2, \dots, \mathbf{x}_d\in\Omega$ and $\mathbf{z}\in\mathbb{Z}^2$, the joint distribution of the Gaussian random vector $[\mathcal{P}(\mathbf{x}_1 + \mathbf{z}), \mathcal{P}(\mathbf{x}_2 + \mathbf{z}), \dots, \mathcal{P}(\mathbf{x}_d + \mathbf{z})]$ is independent of $\mathbf{z}$.
\end{definition}

\begin{definition}\label{def:isotropic}
    A stationary Gaussian random field $\mathcal{P}$ is called \textbf{isotropic}, if its covariance function $\gamma_p(\mathbf{x}, \mathbf{y})$ only depends on the relative distance of pixels $\mathbf{x}$ and $\mathbf{y}$, i.e., $\gamma_p(\mathbf{x} - \mathbf{y}) = \gamma_p(\|{\mathbf{x}}-\mathbf{y}\|_2)$.
\end{definition}

An immediate consequence of Definition \ref{def:stationary} is that the distribution of the $d$-dimensional image patch is independent of the choice of the patch center, i.e. $\vec{\mathcal{P}}(\mathbf{x})\sim\mathcal{N}\left(\vec{\mu}_p, \Sigma_p\right)$ for all $\mathbf{x}\in\Omega$. The patch response $\mathcal{R}(\mathbf{y},\mathbf{x})$ involves the observation of two patches $\vec{\mathcal{P}}(\mathbf{x})$ and $\vec{\mathcal{P}}(\mathbf{y})$. The mutual distribution of the two patches follows
\begin{align}
    \begin{pmatrix}
        \vec{\mathcal{P}}(\mathbf{x})\\
        \vec{\mathcal{P}}(\mathbf{y}) 
    \end{pmatrix}\sim \mathcal{N}\left(
    \begin{pmatrix}
        \vec{\mu}_p\\
        \vec{\mu}_p
    \end{pmatrix},
    \begin{pmatrix}
        \Sigma_p & \Sigma_{\mathrm{c}}(\tau)\\
        \Sigma_{\mathrm{c}}^T(\tau) & \Sigma_p
    \end{pmatrix}
    \right)\notag
\end{align}
where the $d\times d$ covariance matrix $\Sigma_{\mathrm{c}}(\tau) = \Cov(\vec{\mathcal{P}}(\mathbf{x}), \vec{\mathcal{P}}(\mathbf{y}))$ only depends on the relative distance $\tau = \|\mathbf{y} - \mathbf{x}\|_2$ of pixels $\mathbf{x}, \mathbf{y}$, as a consequence of Definition \ref{def:isotropic}. The entries of the covariance function is given by \eqref{eq:exp square covariance}, i.e., 
\[
    \Sigma_{\mathrm{c}}(\tau)[i,j] = \sigma_p^2\exp\left(-\frac{\tau_{i,j}^2}{2l_p^2}\right)
\]
where $\tau_{i,j}$ denotes the relative distance of $i$'th pixel in $\vec{\mathcal{P}}(\mathbf{x})$ and $j$'th pixel in $\vec{\mathcal{P}}(\mathbf{y})$. Let $d = (2r+1)^2$ as in section \ref{sec:method}, and we assume $\tau > 2\sqrt{2}r$, which guarantees that $\vec{\mathcal{P}}(\mathbf{x})$ and $\vec{\mathcal{P}}(\mathbf{y})$ do not overlap, hence $\tau_{i,j} >\tau - 2\sqrt{2}r$ for all $i,j\in [1, (2r+1)^2] \cap \mathbb{Z}$. This leads to an upper bound of the Frobenius norm of the covariance matrix $\Sigma_{\mathrm{c}}$:
\begin{align}
    \|\Sigma_{\mathrm{c}}(\tau)\|_F 
   \;\;  = \;\; \sqrt{\sum_{i,j=1}^{(2r+1)^2}\sigma_p^4 \exp{\left(-\frac{\tau_{i,j}^2}{l_p^2}\right)} }     \;\; \leq \;\; \sigma_p^2 (2r+1)^2\exp\left(-\frac{(\tau - 2\sqrt{2}r)^2}{2l_p^2}\right).\label{eq:frobenius bound}
\end{align}
Fixing $r$, the cross term $\Sigma_{\mathrm{c}}(\tau)\to O_d$ as $\tau\to\infty$, where $O_d$ is $\mathbb{R}^{d\times d}$ null matrix. Comparing \eqref{eq:exp square covariance} and \eqref{eq:frobenius bound}, the decaying rate of correlation of the image patches is consistent with the rate of the pixel-wise covariance function $\gamma(\tau)$.

The conditional distribution of $\vec{\mathcal{P}}(\mathbf{y})$ with respect to $\vec{\mathcal{P}}(\mathbf{x})$ is again multivariate Gaussian, which is fully determined by its mean and variance functions
\begin{align}
    \vec{\mu}_p(\mathbf{y};\mathbf{x}) &= \vec{\mu}_p + \Sigma_{\mathrm{c}}^T(\tau)\Sigma_p^{-1}(\vec{\mathcal{P}}(\mathbf{x})-\vec{\mu}_p)\label{eq:conditional expectation},\\
    \Sigma_p(\mathbf{y};\mathbf{x}) &= \Sigma_p - \Sigma_{\mathrm{c}}^T(\tau)\Sigma_p^{-1}\Sigma_{\mathrm{c}}(\tau).\label{eq:conditional variance}
\end{align}
Combining with \eqref{eq:frobenius bound}, 
$\vec{\mu}_p(\mathbf{y};\mathbf{x})$ and $\Sigma_p(\mathbf{y};\mathbf{x})$ converge to $\vec{\mu}_p$ and $\Sigma_p$ as $\tau\to\infty$, i.e,
\begin{align*}
    \lim_{\tau\to\infty} \|\vec\mu_p(\mathbf{y};\mathbf{x}) - \vec{\mu}_p\|_F = 0,\quad
    \lim_{\tau\to\infty} \|\Sigma_p(\mathbf{y};\mathbf{x}) - \Sigma_p\|_F = 0.
\end{align*}

\subsection{Characteristics of the patch response\label{sec:characteristics}}

In the following, we provide the main results of the section. In order to compute the expectation of $\mathcal{R}(\mathbf{y};\mathbf{x})$, we need the following lemma:
\begin{lemma}[Expectation of quadratic form \cite{seber2003linear}]\label{lemma: expectation of random vector}
Let $\vec{\mathcal{P}}$ be a $d\times1$ random vector with mean $\vec{\mu}_p$ and variance $\Sigma_p$, and let $A$ be an $d\times d$ symmetric matrix. Then
    \begin{align*}
        \mathbb{E}(\vec{\mathcal{P}}^TA\vec{\mathcal{P}}) = \vec{\mu}_p^TA\vec{\mu}_p + \mathrm{tr}(A\Sigma_p)
    \end{align*}
    where $\mathrm{tr}(\cdot)$ is the trace operator.
\end{lemma}

\begin{theorem}\label{thm: expectation of R}
    Let the random field $\mathcal{P}$ be defined as in Definition \ref{def:Gaussian random field}, equipped with the covariance function \eqref{eq:exp square covariance}. Then the patch response $\mathcal{R}(\mathbf{y};\mathbf{x}) = \frac{1}{d}\|\vec{\mathcal{P}}(\mathbf{y}) - \vec{\mathcal{P}}(\mathbf{x})\|_2^2$, where $d = (2r+1)^2$, has expectation
    \begin{align}\label{eq:expectation of R}
        \mathbb{E}\left(\mathcal{R}(\mathbf{y};\mathbf{x})\right) = 2\sigma_p^2\left(1-\exp(-\frac{\tau^2}{2l_p^2})\right).
    \end{align}
\end{theorem}
The proof is presented in Appendix \ref{app:proof of T1}.  Theorem \ref{thm: expectation of R} describes the expectation of $\mathcal{R(\mathbf{y};\mathbf{x})}$ when 
the patch centered at location $\mathbf{y}$ is drawn from $\mathcal{P}$.  

When it is not, i.e. comparing two different textures, let $\vec{\mathcal{Q}}(\mathbf{y}) \sim \mathcal{N}\left(\vec{\mu}_q, \Sigma_q\right)$ be another random field independent from $\mathcal{P}$, where the covariance function is given as
\begin{align*}
    \Cov(\mathcal{Q}(\mathbf{x}_1), \mathcal{Q}(\mathbf{y}_2)) = \gamma_q(\mathbf{x}_1, \mathbf{x}_2) = \sigma_q^2\exp\left(-\frac{\|\mathbf{x}_1 - \mathbf{x}_2\|_2^2}{2l_q^2}\right)
\end{align*}
for some $\sigma_q, l_q >0$. If the patch $\vec{\mathcal{P}}(\mathbf{x})$ is observing $\vec{\mathcal{Q}}(\mathbf{y})$, we simply have
\begin{align*}
    \vec{\mu}_q(\mathbf{y};\mathbf{x}) = \vec{\mu}_q, \quad  \Sigma_q(\mathbf{y};\mathbf{x}) = \Sigma_q,
\end{align*}
since $\Cov(\mathcal{P}(\mathbf{x}), \mathcal{Q}(\mathbf{y})) = 0$.  The expectation of the patch response is given as
\begin{align}
   \mathbb{E}\left(\frac{1}{d}\|\vec{\mathcal{P}}(\mathbf{x}) - \vec{\mathcal{Q}}(\mathbf{y})\|_2^2\right) &= \frac{1}{d}\mathbb{E}_\mathbf{x}\left(\mathbb{E}_{\mathbf{y}|\mathbf{x}}\left(\vec{\mathcal{P}}(\mathbf{x})^T\vec{\mathcal{P}}(\mathbf{x}) - 2\vec{\mathcal{P}}(\mathbf{x})^T\vec{\mathcal{Q}}(\mathbf{y}) + \vec{\mathcal{Q}}(\mathbf{y})^T\vec{\mathcal{Q}}(\mathbf{y})\right)\right)\notag\\
   &= \frac{1}{d}\left(\vec{\mu}_p^T\vec{\mu}_p + \mathrm{tr}(\Sigma_p) - 2\vec{\mu}_p^T\vec{\mu}_q + \vec{\mu_q}^T\vec{\mu_q} + \mathrm{tr}(\Sigma_q)\right)\notag\\
   &= \frac{1}{d}\left(\|\vec{\mu}_p - \vec{\mu}_q\|_2^2 + \mathrm{tr}(\Sigma_p) + \mathrm{tr}(\Sigma_q)\right) 
   = (\mu_p - \mu_q)^2 + \sigma_p^2 + \sigma_q^2 \label{eq:expectation of R_forQ}.
\end{align}
Suppose there are two textures $\mathcal{P}, \mathcal{Q}$ in $\mathcal{B}_R(\mathbf{x})$ while the patch in $\mathcal{B}_r(\mathbf{x})$ is drawn from $\mathcal{P}$.  Texture edge can be detected if the quantities \eqref{eq:expectation of R} and \eqref{eq:expectation of R_forQ} differs, preferably significantly differs.  This difference is described by 
\begin{align}
    \mathrm{diff}(\tau) &= \abs{\mathbb{E}\left(\frac{1}{d}\|\vec{\mathcal{P}}(\mathbf{y}) - \vec{\mathcal{P}}(\mathbf{x})\|_2^2 - \frac{1}{d}\|\vec{\mathcal{P}}(\mathbf{x}) - \vec{\mathcal{Q}}(\mathbf{y})\|_2^2\right)}\notag\\
    &= \abs{(\mu_p - \mu_q)^2 + (\sigma_p^2 - \sigma_q^2) - 2\sigma_p^2\exp{(-\frac{\tau^2}{2l_p^2})}}. \label{eq:diffOfResponse}
\end{align}
Notice that this difference \eqref{eq:diffOfResponse} is a function of $\tau$.  For larger $\tau$, this separation is clearer, yet, more likely $\vec{\mathcal{P}}(\mathbf{x})$ may  encounter another texture $\vec{\mathcal{Q}}(\mathbf{y})$ in a real image when $\mathbf{y}$ is far away from $\mathbf{x}$.  It is rather important to search for edges in $\mathcal{R}(\cdot;\mathbf{x})$ in the region that $\tau$ is large.

\subsection{Stability of the patch response w.r.t. the patch width parameter $r$}

We explore how the value of the patch response $\mathcal{R}(\mathbf{y};\mathbf{x})$ concentrates to its expectation with respect to the patch width parameter $r$.
Since $\mathcal{R}(\mathbf{y};\mathbf{x})$ is a quadratic form of two Gaussian vectors, its distribution can be described by a variation of the $\chi^2$ distribution \cite{Mathai1992quadratic}. 

\begin{theorem} \label{thm: var of R}
    Let $\mathcal{P}$ be defined as in Theorem \ref{thm: expectation of R}, then the patch response $\mathcal{R}(\mathbf{y};\mathbf{x})|_{\vec{\mathcal{P}}(\mathbf{x}) = \vec{v}}$ follows generalized $\chi^2$ distribution \cite{Mathai1992quadratic} in $\vec{\mathcal{P}}(\mathbf{y})$.  Assuming $1/d\|\vec{v} - \vec{\mu}_p\|_2^2 \sim \mathcal{O}(\sigma_p^2)$, the variance of the patch response becomes 
\begin{align*}
    \mathrm{Var}\left(\mathcal{R}(\mathbf{y};\mathbf{x})|_{\vec{\mathcal{P}}(\mathbf{x}) = \vec{v}}\right) \sim \mathcal{O}(\frac{\sigma_p^4}{r^2}).
\end{align*}    

\end{theorem}
\begin{proof}
Fix $\vec{\mathcal{P}}(\mathbf{x}) = \vec{v}$, and denote $\vec{\mathcal{S}} \vcentcolon= \frac{1}{\sqrt{d}}\left(\vec{\mathcal{P}}(\mathbf{y})|_{\vec{\mathcal{P}}(\mathbf{x}) = \vec{v} } - \vec{v}\right)$, then $\vec{\mathcal{S}}$ follows the Gaussian distribution $\vec{\mathcal{S}} \sim \mathcal{N}\left(\vec{\mu}_*, \Sigma_*\right)$, where according to \eqref{eq:conditional variance}, we have
\begin{align*}
    \vec{\mu}_* = \frac{1}{\sqrt{d}}\left(\mathbb{E}\left(\vec{\mathcal{P}}(\mathbf{y})|_{\vec{\mathcal{P}}(\mathbf{x}) = \vec{v} }\right) - \vec{v}\right), \mathrm{~~and~~} \Sigma_* = \frac{1}{d}\left(\Sigma_p - \Sigma_{\mathrm{c}}^T(\tau)\Sigma_p^{-1}\Sigma_{\mathrm{c}}(\tau)\right).
\end{align*}
Let $Q$ be an orthogonal matrix that diagonalize $\Sigma_*$, that is, $Q^T\Sigma_* Q = \text{diag}(\lambda_1, \lambda_2, \dots, \lambda_d) = \Lambda$, where $\lambda_i > 0$ are the eigenvalues of $\Sigma_*$. Define a new random vector 
\[
\vec{\mathcal{U}} = Q^T\Sigma^{-\frac{1}{2}}_*(\vec{\mathcal{S}} - \vec{\mu}_*),
\]
here $\vec{\mathcal{U}}$ is standard Gaussian, i.e., $\vec{\mathcal{U}} \sim \mathcal{N}(\vec{0}, I_{d})$. The observed patch response can be reformulated as
\begin{align}
    \mathcal{R}(\mathbf{y};\mathbf{x})|_{\vec{\mathcal{P}}(\mathbf{x}) = \vec{v}} &= \|\vec{\mathcal{S}}\|_2^2 
     = (\vec{\mathcal{U}} + \vec{b})^T Q^T \Sigma_* Q(\vec{\mathcal{U}}+\vec{b}) = (\vec{\mathcal{U}} + \vec{b})^T \Lambda (\vec{\mathcal{U}}+\vec{b}) 
    = \sum_{j=1}^{d} \lambda_j (\mathcal{U}_j + b_j)^2\label{eq:response distribution},
\end{align}
where $\vec{b} = Q^T\Sigma^{-\frac{1}{2}}_*\vec{\mu}_*$, and $\mathcal{U}_j$ and $b_j$ denote the $j$'th element of vectors $\vec{\mathcal{U}}$, and $\vec{b}$, respectively.   The response \eqref{eq:response distribution} is a weighted sum of squares of $d$ independent Gaussian variables $(\mathcal{U}_j + b_j)\sim\mathcal{N}(b_j,1)$. Each $(\mathcal{U}_j + b_j)^2$ follows noncentral chi-squared distribution $\chi_\nu^2(\delta)$, which is fully described by the degree of freedom $\nu$ and noncentrality parameter $\delta$, and the mean and variance of such distribution is given by $\nu + \delta$ and $2(\nu + 2\delta)$. Specifically, we have $(\mathcal{U}_j + b_j)^2\sim\chi^2_1(b^2_j)$.
The density of the patch response \eqref{eq:response distribution} in general does not  have a closed form \cite{davies1980algorithm}. With $d = (2r+1)^2$, its variance becomes 
\begin{align*}
    \mathrm{Var}\left(\mathcal{R}(\mathbf{y};\mathbf{x})|_{\vec{\mathcal{P}}(\mathbf{x}) = \vec{v}}\right) &= \sum_{j=1}^{d}\lambda_j^2 \mathrm{Var}(\mathcal{U}_j + b_j)^2 
    = 2\sum_{j=1}^{d}\lambda_j^2(1 + 2b_j^2)\notag\\
    &= 2\mathrm{tr}(\Lambda^2) + 4\vec{b}^T\Lambda^2 \vec{b} = 2\mathrm{tr}(\Sigma_*^2) + 4\vec{\mu}_*^T\Sigma_*\vec{\mu}_*\notag\\
    &= \frac{2}{d^2}\left(\mathrm{tr}\left(\Sigma_p(\mathbf{y};\mathbf{x})^2\right) + 2(\vec{\mu}_p - \vec{v})^T \Sigma_p(\mathbf{y};\mathbf{x}) (\vec{\mu}_p - \vec{v})\right)\sim \mathcal{O}(\frac{\sigma_p^4}{d}).
\end{align*}
\end{proof}

In Figure \ref{fig:analysis_experiment}, (a) shows a synthetic image consists of two textures $\mathcal{P}$ (left) and $\mathcal{Q}$ (right) from Brodatz texture images set\footnote{The Brodatz texture image set is obtained from \url{https://sipi.usc.edu/database/}}. Two patches $\vec{\mathcal{P}}(\mathbf{x})$ and $\vec{\mathcal{Q}}(\mathbf{y})$ are  marked with blue and yellow squares correspondingly. (b) and (c) show two patch responses $\mathcal{R}(\cdot; \mathbf{x})$ and $\mathcal{R}(\cdot; \mathbf{y})$. The brightness is proportional to the value of the patch responses. 
This shows that with a suitable patch width parameter $r$, the texture edge is clearly emphasized, which is consistent with the analysis in section \ref{sec:characteristics}. 
The contrast of two textured regions clearly indicates the edge location.

\begin{figure}
    \centering
    \begin{tabular}{ccc}
        (a) & (b) & (c)\\
        \includegraphics[width=0.23\textwidth, angle=180,origin=c]{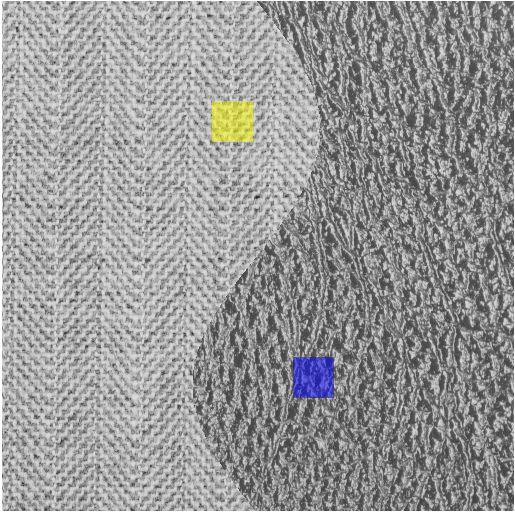} &
        \includegraphics[width=0.23\textwidth, angle=180,origin=c]
        {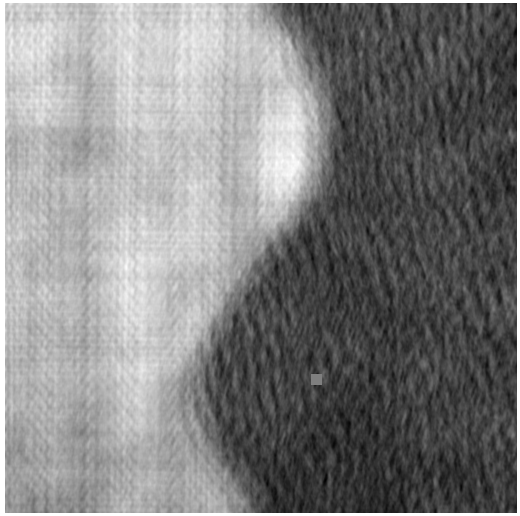} &
        \includegraphics[width=0.23\textwidth, angle=180,origin=c]
        {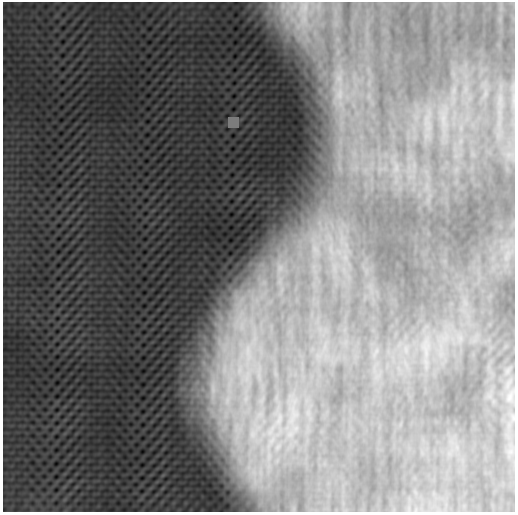}  
    \end{tabular}
    \caption{(a) Synthetic texture image with two textures $\mathcal{P}$ (left) and $\mathcal{Q}$ (right). Two patches $\vec{\mathcal{P}}(\mathbf{x})$ and $\vec{\mathcal{Q}}(\mathbf{y})$ are  marked with blue and yellow. (b) and (c) show two patch responses $\mathcal{R}(\cdot; \mathbf{x})$ and $\mathcal{R}(\cdot; \mathbf{y})$ respectively. Note that the texture edge is clearly emphasized with a suitable patch width parameter $r$.} \label{fig:analysis_experiment}  
\end{figure}

The edge function $V$ is given by the consensus of many patch responses.  For accurate edge detection, these responses from different observer patches should be consistent, that is many patch responses should recognize there is an edge.   This can be measured by the distribution of $\mathbb{E}_{\mathbf{y}|\mathbf{x}}\left(\mathcal{R}(\mathbf{y}; \mathbf{x})\right)$, the expected response from the perspective of patch $\vec{\mathcal{P}}(\mathbf{x})$. 
In Figure \ref{fig:response_pdf} (a), the histograms of the two textures in Figure \ref{fig:analysis_experiment} (a) are presented.  The intensity values of the two textures are heavily overlapped which indicates the challenges of using intensity based method to detect the texture boundaries.  By using the patch based consensus, Figure \ref{fig:response_pdf} (b) and (c) show that as the patch width parameter increases,  the more concentrated the expectation becomes. This is consistent with Theorem \ref{thm: var of R}, thus helping to distinguish two textures. 
Figure \ref{fig:response_pdf} (b) and (c) show the estimated distribution of $\mathbb{E}_{\mathbf{y}|\mathbf{x}}\left(\mathcal{R}(\mathbf{y}; \mathbf{x})\right)$, (b) is assuming the pixel $\mathbf{x}$ is from a random field $\mathcal{P}$, and (c) is assuming the pixel $\mathbf{x}$ is from a random field $\mathcal{Q}$.
Note that $\mathbb{E}_\mathbf{x}\left(\mathbb{E}_{\mathbf{y}|\mathbf{x}}\left(\mathcal{R}(\mathbf{y}; \mathbf{x})\right)\right)$ is given from  \eqref{eq:expectation of R} in Theorem \ref{thm: expectation of R}, if $\mathbf{y}$ is equipped with $\mathcal{P}$, and from \eqref{eq:expectation of R_forQ}, if equipped with $\mathcal{Q}$. Two sets of the distributions are concentrated around estimated expectations, which are computed from the pixel-wise means and variances of the texture images. 
We observe the concentration effect with a larger concentration rate, which is due to the variance difference of two textures.  In particular, we have $\sigma_p > \sigma_q$, and according to Theorem \ref{thm: var of R}, the variance of $\mathbb{E}_{\mathbf{y}|\mathbf{x}}\left(\mathcal{R}(\mathbf{y}; \mathbf{x})\right)$ is $\mathcal{O}(\sigma_p^4/r^2)$ for   Figure \ref{fig:response_pdf} (b) and $\mathcal{O}(\sigma_q^4/r^2)$ for   Figure \ref{fig:response_pdf} (c). Neither of the textures $\mathcal{P}$ and $\mathcal{Q}$ are strictly stationary nor isotropic, yet our model well-describes the behavior of the patch response.

\begin{figure}
    \centering
    \begin{tabular}{ccc}
        (a) & (b) & (c)\\
    \includegraphics[width=0.3\textwidth]{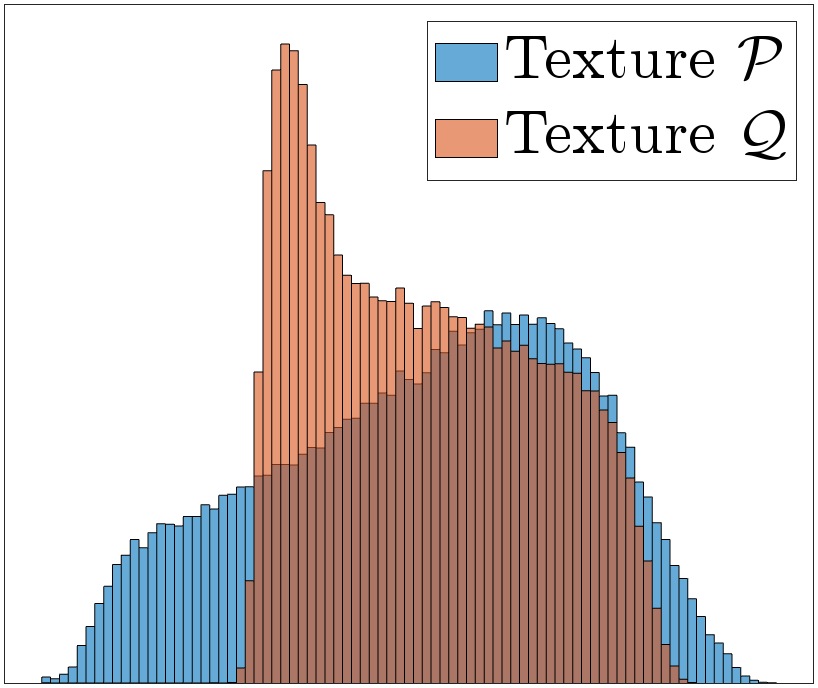} &
    \includegraphics[width=0.3\textwidth]{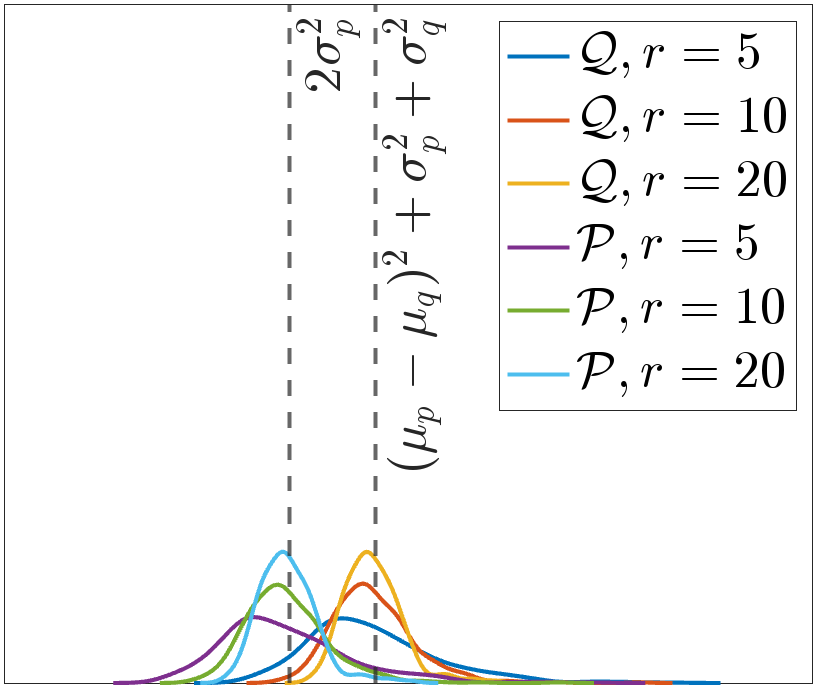} &
    \includegraphics[width=0.3\textwidth]{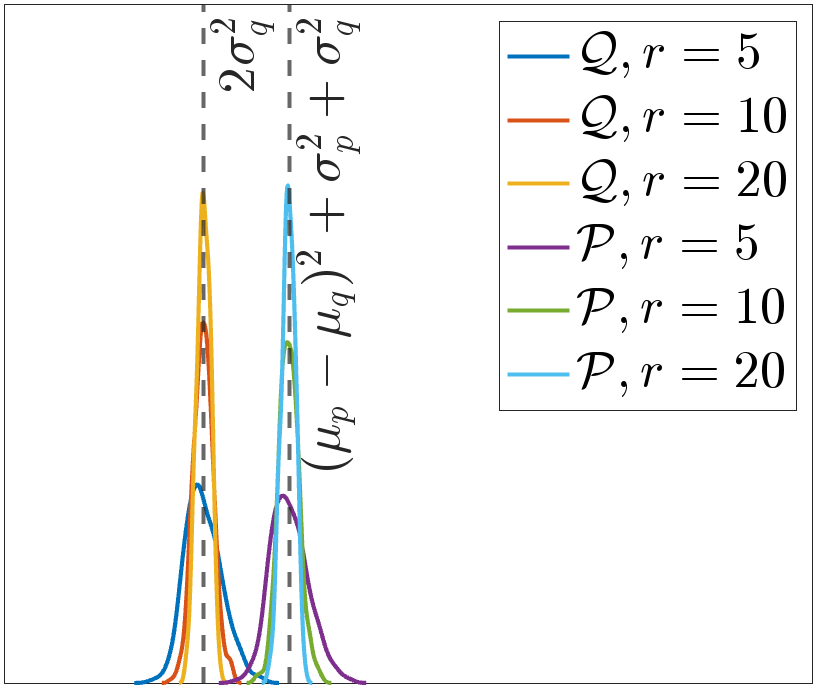}  
    \end{tabular}
    \caption{
    (a) The intensity histograms of two textures in Figure \ref{fig:analysis_experiment} (a). (b) Estimated distributions of the patch response $\mathbb{E}_{\mathbf{y}|\mathbf{x}}\left(\mathcal{R}(\mathbf{y}; \mathbf{x})\right)$ as $r$ increases, assuming the patch centered at $\mathbf{x}$ is sampled from $\mathcal{P}$, and $\mathbf{y}$ from $\mathcal{P}$ or $\mathcal{Q}$ as indicated in the legend. (c) Same as (b) assuming the patch centered at $\mathbf{x}$ is sampled from $\mathcal{Q}$.}
    \label{fig:response_pdf}
\end{figure}

\subsection{The patch width parameter $r$ and edge detection} \label{sec:patchsize}

The quality of edge detection depends on the intensity contrast of the patch response $\mathcal{R}(\mathbf{y};\mathbf{x})$.  This contrast is given by the responses observing $\mathcal{P}$ or $\mathcal{Q}$ by the patch centered at $\mathbf{x}$, and the regularity of  $\mathcal{R}(\mathbf{y};\mathbf{x})$ is related to the choice of $r$. 

For two texture separation, we first assume $R$ is chosen that $\mathcal{B}_R(\mathbf{x})$ contains two different textures $\mathcal{P}$ and $\mathcal{Q}$ with a fixed pixel $\mathbf{x}$ away from any texture boundary.  We use the squared Hellinger distance \cite{hellinger1909neue} of two probability density functions $f_1, f_2$ to compare the two different patch responses:
\begin{align}\label{eq:hellinger}
	\mathcal{H}^2(f_1, f_2) = 1 - \sqrt{\langle f_1, f_2\rangle} \in [0, 1],
\end{align}
here $\langle\cdot, \cdot\rangle$ denotes the inner product. Squared Hellinger distance \eqref{eq:hellinger} is a bounded metric that measures the similarity of the probability density functions $f_1, f_2$ in terms of  the overlap. 
In Figure \ref{fig:Hellinger}, the blue curve  indicates the squared Hellinger distance of the patch responses of observing textures $\mathcal{P}$ and $\mathcal{Q}$ from the perspective of patch $\vec{\mathcal{P}}(\mathbf{x})$ 
and the red curve is from the perspective of patch $\vec{\mathcal{Q}}(\mathbf{x})$.  These curves represents the differences of the density functions shown in Figure \ref{fig:response_pdf} (b) and (c).  As $r$ increases, two responses get better separated in Figure \ref{fig:response_pdf} (b) and (c), which is represented as the increasing value of squared Hellinger distance. 
The growth of two blue and red curves are different as $r$ increases, which is due to the difference in the variance of two textures $\mathcal{P}$ and $\mathcal{Q}$ in Figure \ref{fig:response_pdf}.  
The horizontal dash line in Figure \ref{fig:Hellinger} shows a wide range of $r$ which gives the separation of two textures. 
\begin{figure}
    \centering
\includegraphics[width=0.45\textwidth]{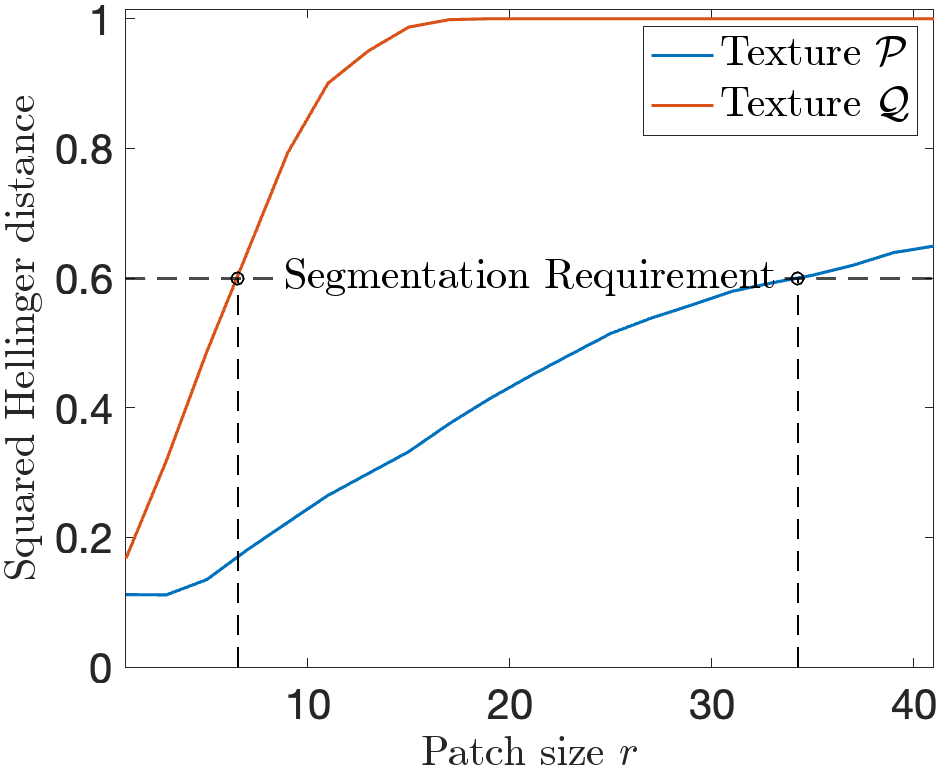}
    \caption{Change of distance of distribution function of $\mathbb{E}_{\mathbf{y}|\mathbf{x}}(\mathcal{R}(\mathbf{y};\mathbf{x}))$ with respect to the patch width parameter $r$. Blue line assumes the observer $\mathbf{x}$ being equipped with $\mathcal{P}$, while the red line assumes $\mathbf{x}$ being equipped with $\mathcal{Q}$. The horizontal dashed line indicates the required minimal distance for two textures to be distinguished by the segmentation model.}
    \label{fig:Hellinger}
\end{figure}
In practice, the patch width parameter $r$ only needs to meet segmentation requirement of one of two adjacent textures.

\section{Numerical Details} \label{sec:Numerical}

We summarize the proposed method in Algorithm \ref{alg:edge_detection} which includes following   modifications for an efficient computation.  
\begin{algorithm2e}
	\SetKwInOut{Input}{Input}\SetKwInOut{Output}{Output}
	\Input{The given image $U$, the patch width parameter $r$, the large comparison region width parameter $R$ , the regularity parameter $\lambda$ 
 for the segmentation model \eqref{eq:multiphase}, and  the parameter $\delta$ for  modification \eqref{eq:numerical r}.  }
    Initialize $V$ as a zero matrix of the size of $U$\;
	\For{$\mathbf{x}\in\Omega$}{
        \For{$\mathbf{y}\in\mathcal{B}_R(\mathbf{x})$}{
        Compute $\mathcal{R}(\mathbf{y};\mathbf{x})$ in  \eqref{e:response}, and modify to $\hat{\mathcal{R}}(\cdot;\mathbf{x})$ as in  (\ref{eq:numerical r})\;
        }
        Compute $\mathcal{W}(\cdot; \mathbf{x})$ from the segmentation (\ref{eq:w}), and modify to $\hat{\mathcal{W}}(\cdot; \mathbf{x})$ as in (\ref{eq:numerical w})\;
	Update $V\vert_{\mathcal{B}_R(\mathbf{x})} \leftarrow V\vert_{\mathcal{B}_R(\mathbf{x})} + \frac{1}{(2R+1)^2}\hat{\mathcal{W}}(\cdot;\mathbf{x})$\;
	}
	\Output{$V$ the edge function of the given image $U$. }
	\caption{Texture Edge Detection by Patch consensus\label{alg:edge_detection}}
\end{algorithm2e}

First, when computing the patch response $\mathcal{R}(\cdot;\mathbf{x})$, if two points $\mathbf{x}$ and $\mathbf{y}$ are very close, i.e. $\mathbf{y}$ is inside  $\mathcal{B}_\delta(\mathbf{x})$ for $\delta$ small, the patches $\vec{\mathcal{P}}(\mathbf{y})$ and $\vec{\mathcal{P}}(\mathbf{x})$ overlapped in most parts.  This results in unwanted singularity around the center of $\mathcal{R}(\cdot;\mathbf{x})$.  We remove this center singularity with a local average:
\begin{align}
    \hat{\mathcal{R}}(\mathbf{y};\mathbf{x}) = 
    \begin{cases}
        \frac{1}{\abs{\mathcal{B}_R(\mathbf{x})/\mathcal{B}_\delta(\mathbf{x})}}\sum_{\textbf{z} \in\mathcal{B}_R(\mathbf{x})/\mathcal{B}_\delta(\mathbf{x})} \mathcal{R}(\textbf{z};\mathbf{x}) &\quad\text{if } \|\mathbf{y}-\mathbf{x}\|_{\infty} \leq \delta,\\
        \mathcal{R}(\mathbf{y};\mathbf{x}), &\quad\text{otherwise}.
    \end{cases}\label{eq:numerical r}
\end{align}
The patch response $\mathcal{R}(\mathbf{y};\mathbf{x})$ in the subdomain $\mathcal{B}_\delta(\mathbf{x})$ is replaced by the average over its complement $\sum_{\textbf{z} \in\mathcal{B}_R(\mathbf{x})/\mathcal{B}_\delta(\mathbf{x})} \mathcal{R}(\textbf{z};\mathbf{x})$. In practice, we choose  $\delta = 5$ when the patch width parameter $r\in[10, 30]$.

Secondly, when $\mathcal{B}_R(\mathbf{x})$ is close to, but not overlapped with, any texture edge, the patch centered at some pixel $\mathbf{y}\in\mathcal{B}_R(\mathbf{x})$ may still see the texture edge outside $\mathcal{B}_R(\mathbf{x})$.  This can cause the local edge function  $\mathcal{W}(\mathbf{y};\mathbf{x})$ to report a false positive edge inside $\mathcal{B}_R(\mathbf{x})$, and give thick and blurry edge on $V$. We make the local edge function  $\mathcal{W}(\mathbf{y};\mathbf{x})$ to only respond  within $\mathcal{B}_R(\mathbf{x})$, by the following modification 
\begin{align}\label{eq:numerical w}
    \hat{\mathcal{W}}(\mathbf{y};\mathbf{x}) = 
        \begin{cases}
            \mathcal{W}(\mathbf{y};\mathbf{x}) &\quad\text{if }\mathrm{d}(\mathbf{y},\partial B_R(\mathbf{x})) > r,\\
            0 &\quad\text{otherwise,}
        \end{cases}
\end{align}
where $\mathrm{d}(\mathbf{y},\partial B_R(\mathbf{x}))$ is the distance of pixel $\mathbf{y}$ to the boundary of $\mathcal{B}_R(\mathbf{x})$, i.e.,
\[
    \mathrm{d}(\mathbf{y},\partial B_R(\mathbf{x})) = \min\left\{\|\mathbf{y} - \mathbf{z}\|_\infty ~|~ \mathbf{z}\in\partial \mathcal{B}_R(\mathbf{x})\right\}.
\]
 
Thirdly, we bound the number of phases to be $K\in\{1, 2\}$ in the segmentation step.  When $K = 1$, the energy \eqref{eq:multiphase} reduces to the variance of the given image. The effect of the parameter $\lambda$ can be interpreted as a threshold on the segmentation model to give one or two phases.  We set $\lambda$ = 0.01 to 0.05, when normalized patch response $\mathcal{R} \in [0,1]$ is used.  When the given image range is $U \in [0,255]$ and the patch response is not normalize, we use $\lambda$ = 450 to 1,000.    
In Figure \ref{fig:Hellinger}, the horizontal dashed line represents the distance threshold for the two textures to be separated, i.e. the segmentation model to choose $K=2$. The $\lambda$ controls the regularity for the local edge function  $\mathcal{W}(\mathbf{y};\mathbf{x})$, and efficiently reduce the unwanted edge detected. With $\lambda$ fixed, textures requires different patch width parameters $r$ to find an edge (if there is one).

\section{Numerical Experiments}\label{sec:experiments}

In this section, we present numerical results exploring different aspects of the proposed model. 
First,  Figure \ref{fig:A_response} represents the procedure of the proposed method. In the center figure, for each point $\mathbf{x}$, yellow boxes show the local patch $\vec{\mathcal{P}}(\mathbf{x})$ with $r=3$, and the blue boxes show the patch responses $\mathcal{R}(\mathbf{y};\mathbf{x})$ in $\mathcal{B}_R (\mathbf{x})$. 
For each zoomed location, we present the yellow local patch $\vec{\mathcal{P}}(\mathbf{x})$, patch response  $\mathcal{R}(\mathbf{y};\mathbf{x})$ and the local edge function $\mathcal{W}(\mathbf{y};\mathbf{x})$. 
For $\mathbf{x}_1$, $\mathbf{x}_3$, and $\mathbf{x}_4$, two regions are identified and an edge is found between two textures. For $\mathbf{x}_6$, two edges are found separating the patch response to three regions, here two of these three regions represents the same textured region.  
Notice for $\mathbf{x}_2$ and $\mathbf{x}_5$, although textures are changing and patch response shows some textures, they are identified to be the same textured regions and no edges are found.  
\begin{figure}
\centering
\includegraphics[width = 0.84\textwidth]{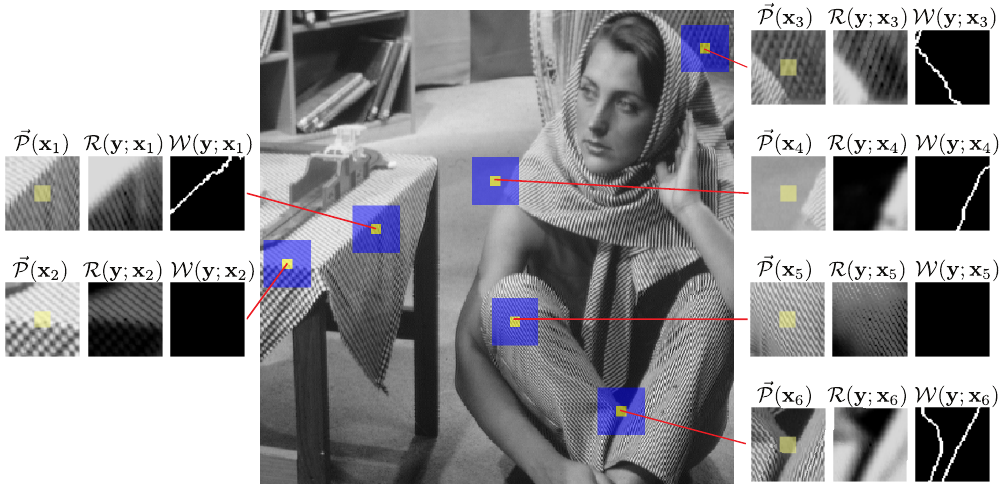}
\caption{For various locations $\mathbf{x}$, yellow boxes show the local patch $\mathcal{P}(\mathbf{x})$ with $r=3$, and the blue boxes show the patch responses $\mathcal{R}(\mathbf{y};\mathbf{x})$ in $\mathcal{B}_R$.  The zoomed-in images also show the local edge function $\mathcal{W}(\mathbf{y};\mathbf{x})$. 
For $\mathbf{x}_1$, $\mathbf{x}_3$, $\mathbf{x}_4$ and $\mathbf{x}_6$, edges are clearly found, while for $\mathbf{x}_2$ and $\mathbf{x}_5$, although patch responses show some textures, they are identified to be the same textured regions and no edges are found. }
\label{fig:A_response}
\end{figure}

\subsection{Real images with texture}

We represent the texture edge detection result for real textured images, and show comparison with Canny edge detection \cite{canny1986computational}.
In Figure \ref{fig:barbara}, TEP finds texture and object boundaries without finding edges within textures.  Zoom-in of the red and the yellow boxes in (a) are presented in (d)-(g), where (d) and (f) shows how TEP $V(\mathbf{x})$ only finds the boundary of the textures. 
In (d), TEP result considers the checkerboard texture as one region, and is able to detect the subtle transitions at the corner of the table. The train rail is considered as an entity, despite the track lines in (d), while, the Canny edge detection in (e) finds sharp gradient change as edges, and finds the edges of the checkerboard pattern also.   
In (f), notice that the shades caused by wrinkles are ignored by the proposed model, while it is captured by Canny edge detection in (g).  For TEP, $r = 3,  R = 35$, and $\lambda = 1000$ are used, while for the Canny edge detection, we used $(0.04, 0.1)$ for hysteresis thresholding and $\sigma = 2$ for Gaussian blurring. 
\begin{figure}
    \centering
    \begin{tabular}{ccc}
   (a) & (b) & (c) \\ 
{\includegraphics[width=0.25\textwidth]{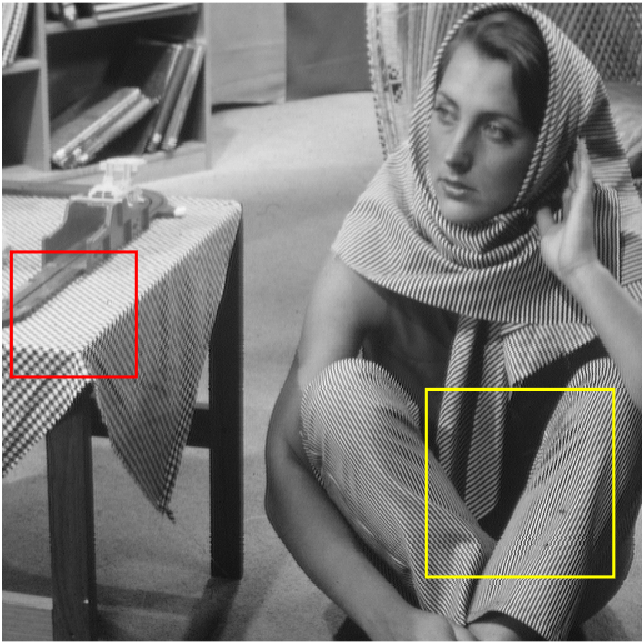}} &
{\includegraphics[width=0.25\textwidth]{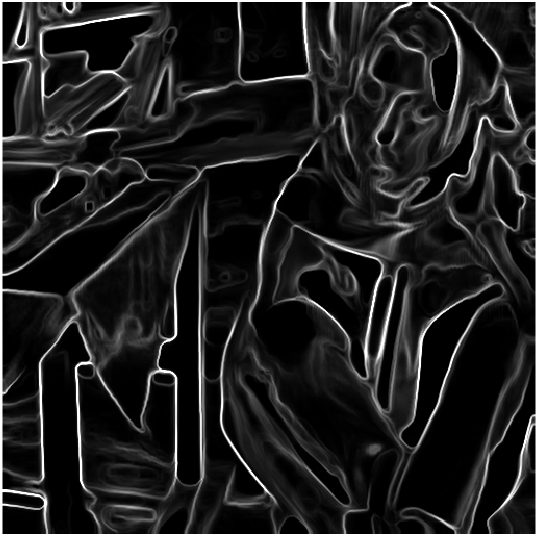}} &    
{\includegraphics[width=0.25\textwidth]{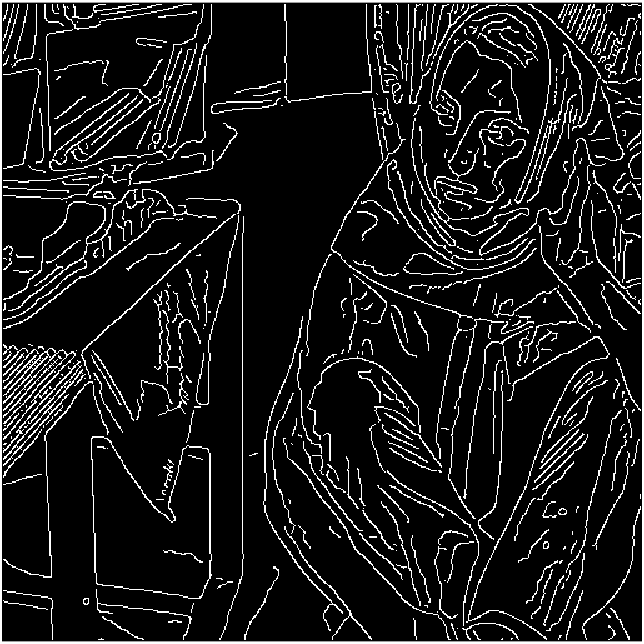}}\\
\end{tabular}
  \begin{tabular}{cccc}
  (d) & (e) & (f) & (g) \\{\includegraphics[width=0.18\textwidth]{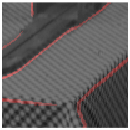}} &    
{\includegraphics[width=0.18\textwidth]{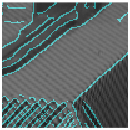}} &   
{\includegraphics[width=0.18\textwidth]{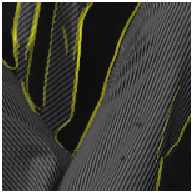}} &    
{\includegraphics[width=0.18\textwidth]{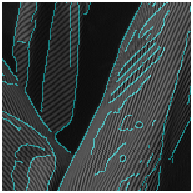}}
    \end{tabular}
    \caption{(a) The given image. (b) TEP result $V(\mathbf{x})$. (c) Canny edge detection. (d) and (e) are zoom-in of the red box, and  (f) and (g) are that of the yellow box in (a). TEP edges $V(\mathbf{x})$ are red edge in (d) and yellow edge in (f). (e) and (g) show Canny edge detection in Cyan. TEP finds the texture edges without finding the edges inside one texture.   }
    \label{fig:barbara}
\end{figure}
In Figure \ref{fig:worm}, first two rows (a)-(g), the worm details are understood as texture in (d). For TEP, $r = 5, R = 20$, and $\lambda = 800$ are used, and for Canny edge detector, threshold parameters $\{0.04, 0.1\}$ and $\sigma=2$ for Gaussian filter are used.
In Figure \ref{fig:worm} last row, the details of the hair is understood as texture by TEP, while Canny edge detection finds the details.  For TEP, $r = 5, R = 30$, and $\lambda = 450$ are used, and for Canny edge detector, threshold parameters $\{0.12, 0.3\}$ and $\sigma=2$ for Gaussian filter are used. 
TEP consistently represents the region better even for textures with complicated and large scale patterns.
\begin{figure}
    \centering
    \begin{tabular}{ccc}
        (a)&(b)&(c)\\
        \includegraphics[width=0.3\textwidth]{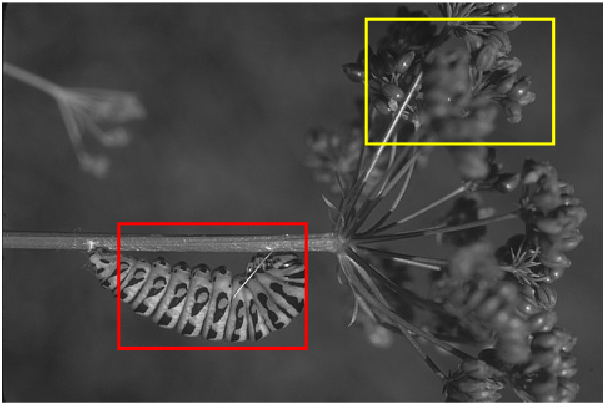}&               
        \includegraphics[width=0.3\textwidth]{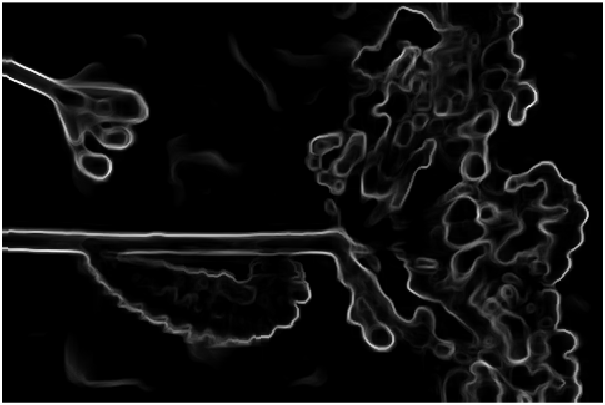}&
        \includegraphics[width=0.3\textwidth]{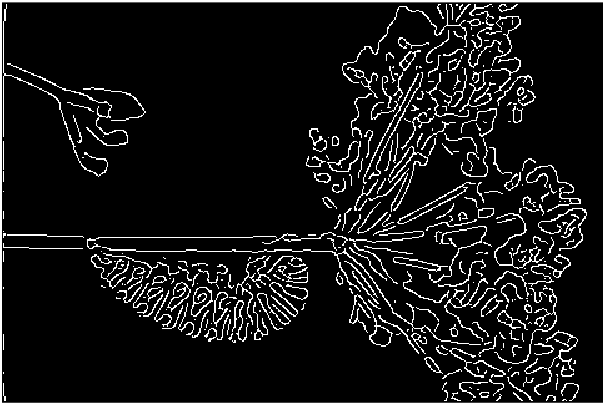}
    \end{tabular}
    \begin{tabular}{cccc}
        (d)&(e)&(f)&(g)\\
        \includegraphics[width=0.2\textwidth]{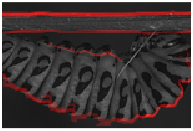}&
        \includegraphics[width=0.2\textwidth]{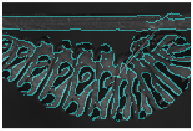}&
        \includegraphics[width=0.2\textwidth]{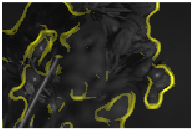}
        & \includegraphics[width=0.2\textwidth]{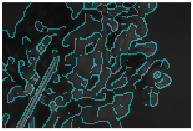}
    \end{tabular}    
        \begin{tabular}{ccc}
        (h)&(i)&(j)\\
       \includegraphics[width=0.25\textwidth]{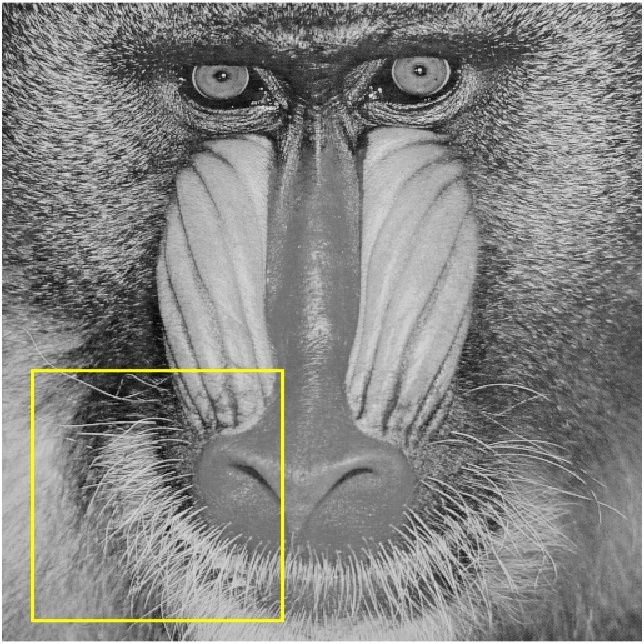}&
         \includegraphics[width=0.25\textwidth]{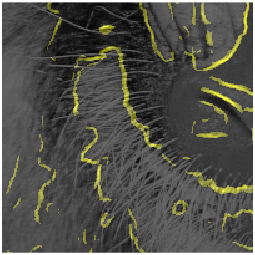} & 
        \includegraphics[width=0.25\textwidth]{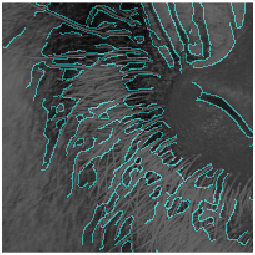} 
    \end{tabular}    
    \caption{(a) The given image. (b) TEP result $V(\mathbf{x})$. (c) Canny edge detection. (d) and (e) are zoom-in of the red box, and  (f) and (g) are that of the yellow box in (a). TEP edges $V(\mathbf{x})$ are red edge in (d) and yellow edge in (f). (e) and (g) show Canny edge detection in Cyan.  Texture of the worm using TEP is clearly understood as texture in (d) and red edge finds the boundary of the worm. (h) The given image.  (i) and (j) are zoom-in of the yellow box in (h).  (i) TEP result $V(\mathbf{x})$ in yellow. (j) Canny edge detection in cyan.  TEP finds hairy region as one texture.}
    \label{fig:worm}
\end{figure}

TEP is a training-free method for texture edge detection.  Yet, in Figure \ref{fig:DP_Compare}, we present images from the Berkeley segmentation dataset BSDS500, and compare with the state-of-the-art machine learning model, Edge Detector with Transformer (EDTER) \cite{pu2022edter} as an example.    
Since the Berkeley Segmentation dataset  for edge detection was published \cite{MartinFTM01}, it has been a benchmark for contour detection, especially in machine learning community \cite{xie15hed, liu2017richer, pu2022edter, he2019bi}.  These methods are trained with color images with ground-truth data provided by human experts \cite{MartinFTM01} that these methods aim at object detection.
On the other hand, We apply TED only on gray scale images without any a priori knowledge of the image.  TEP detects local texture edges, and this is not an object detection method.  Even then,  in Figure \ref{fig:DP_Compare}, TEP shows good edge detection and provides comparable results to the deep learning model.  
In the first row images, TEP and EDTER both finds large scale region with bricks (while Canny edge detection finds details of the bricks).  TEP gives different strength to the edge, some parts are weaker edges than others, while EDTER gives the same strength, since it is object oriented contour detection.   
In the second and third rows, TEP edge is closer to the given image, grouping different texture correctly, and TEP finds irregular texture boundary. In the last row, while TEP finds more details of the dress, EDTER is simplified, and TEP sees the texture of the flood and finds the edge of the texture, while EDTER finds the edges in the floor tiles.  With texture edge detection, TEP can give comparable good edge detection results. 
\begin{figure}
    \centering
    \begin{tabular}{cccc}
    (a)&(b)&(c)&(d)\\
        \includegraphics[width = 0.2\linewidth]{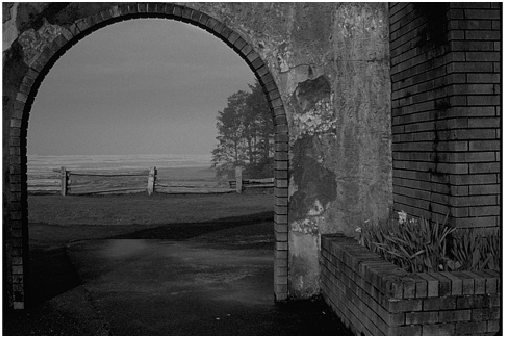} &
        \includegraphics[width = 0.2\linewidth]{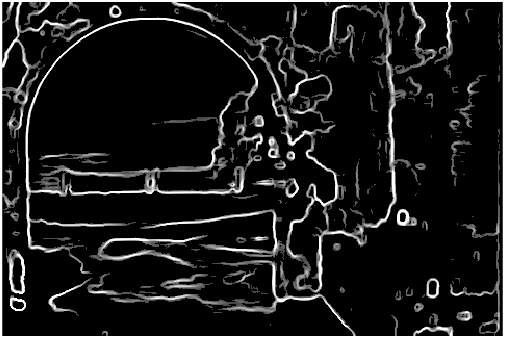} &
       \includegraphics[width = 0.2\linewidth]{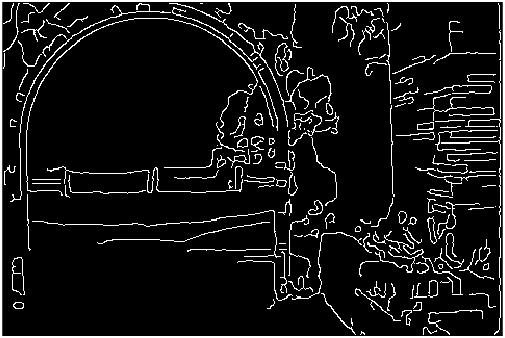} &
        \includegraphics[width = 0.2\linewidth]{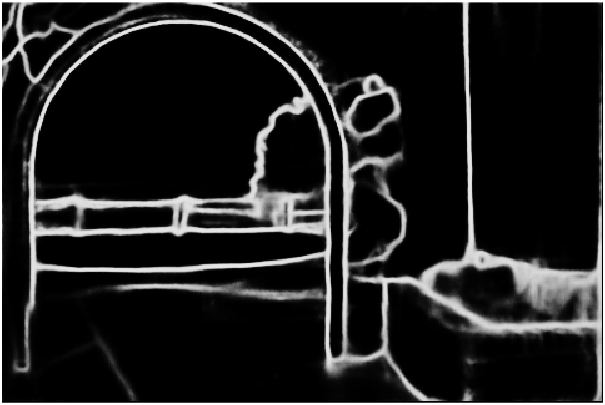}\\
        \includegraphics[width = 0.2\linewidth]{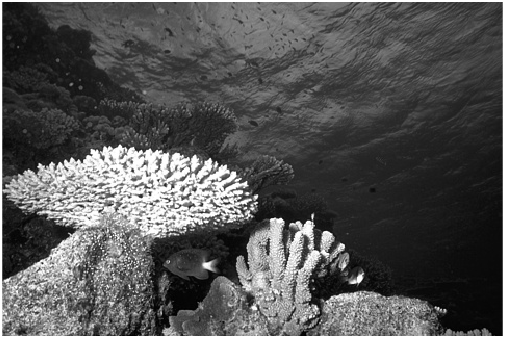} &
        \includegraphics[width = 0.2\linewidth]{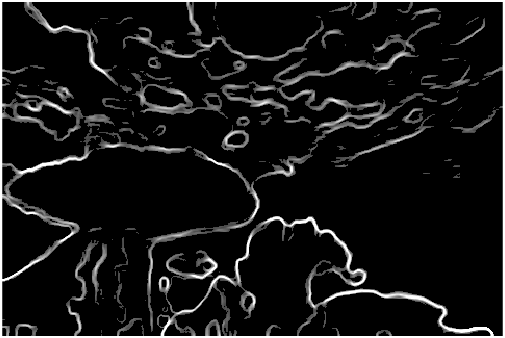} &
         \includegraphics[width = 0.2\linewidth]{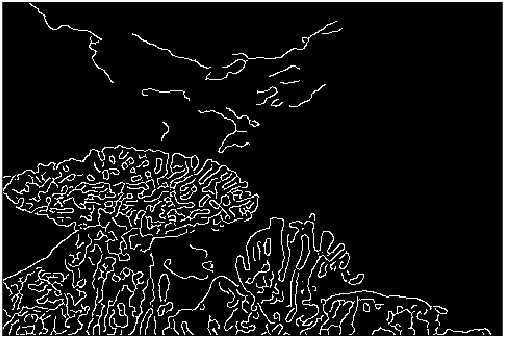} &
        \includegraphics[width = 0.2\linewidth]{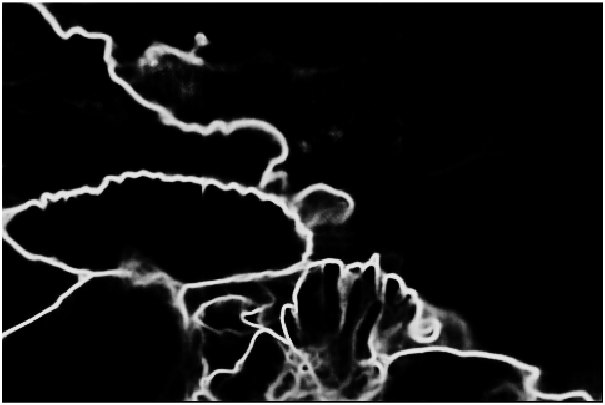} \\
        \includegraphics[width = 0.2\linewidth]{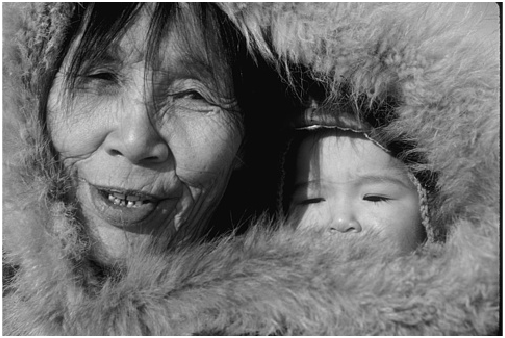} &
        \includegraphics[width = 0.2\linewidth]{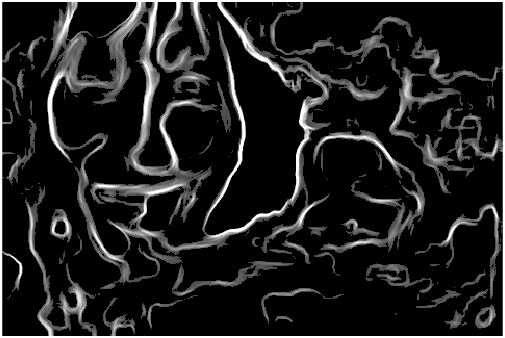} &
         \includegraphics[width = 0.2\linewidth]{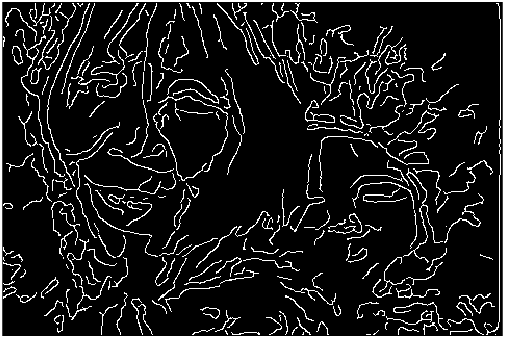} &
        \includegraphics[width = 0.2\linewidth]{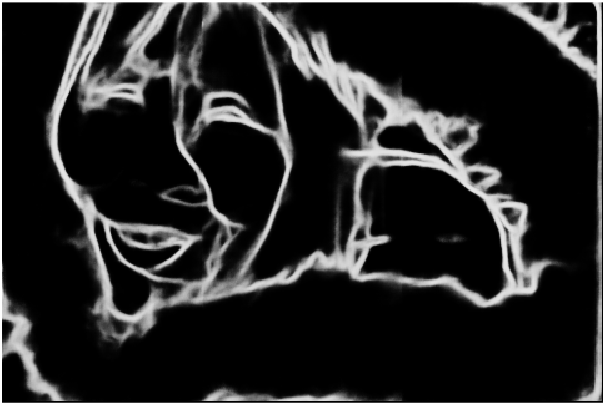} \\
           \includegraphics[width = 0.2\linewidth]{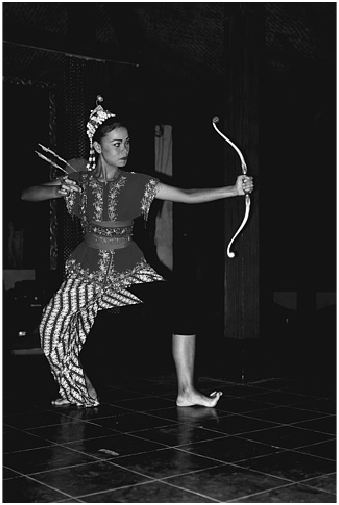} &
        \includegraphics[width = 0.2\linewidth]{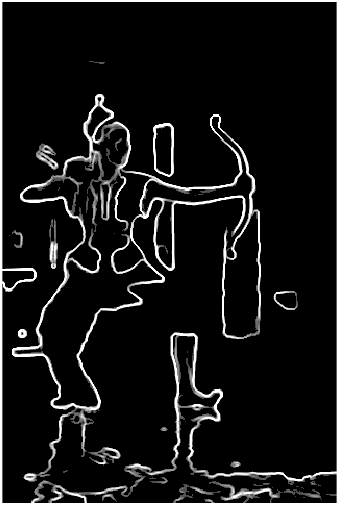} &
       \includegraphics[width = 0.2\linewidth]{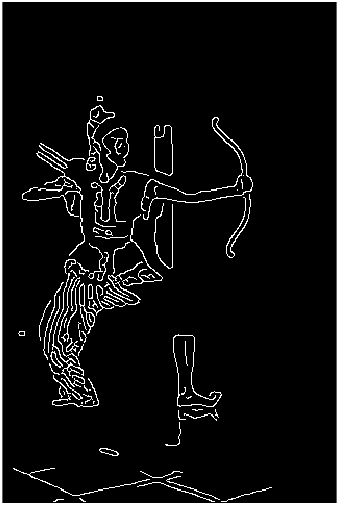} &
        \includegraphics[width = 0.2\linewidth]{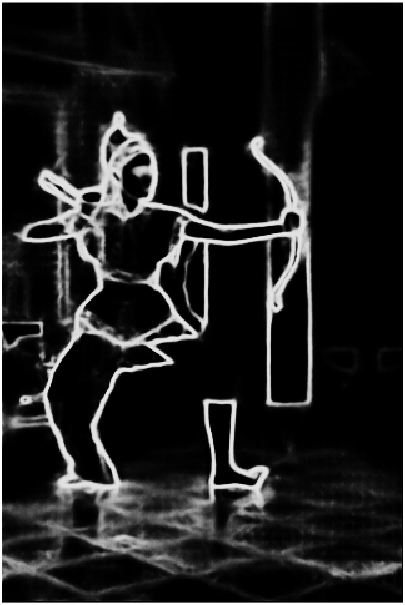}
    \end{tabular}
    \caption{Images from BDSD500 Dataset. First column (a) are the given images. (b) TEP results. (c) Canny edge detection. (d) EDTER results.}
    \label{fig:DP_Compare}
\end{figure}

\subsection{The scale of the texture vs the patch width parameter}

The patch width parameter $r$ can be adjusted to find different scales in the image.  In Figure \ref{fig:scale experiment}, we experiment with image (a) which has different scales of texture for each object.   The background has the smallest texton - the smallest repeating unit in the texture.  The triangle, the circle and the square all have different sizes of texton in increasing order.  From image (b) to (d), the patch width parameter $r$ is increasing from $r=4, 6$ to 8, and as $r$ increases, TEP sees bigger patches as a texture.   In (b) with a small $r$, each texton in the square is identifies as a separate region, since it understands the texture only in the small scale that each texton within the square is understood as a separate region.  In the circle while the edge of the circle is identified, within the circle it also shows some texture details.  In image (d), even the big texton is captured by a large $r$, that all objects clearly shows the texture edge boundary. When $r$ is large enough, TEP ignores the fine details within the textured region. 
\begin{figure}
	\centering
    \begin{tabular}{cccc}
         (a)&(b)&(c)&(d)  \\
        \includegraphics[width=0.22\textwidth]{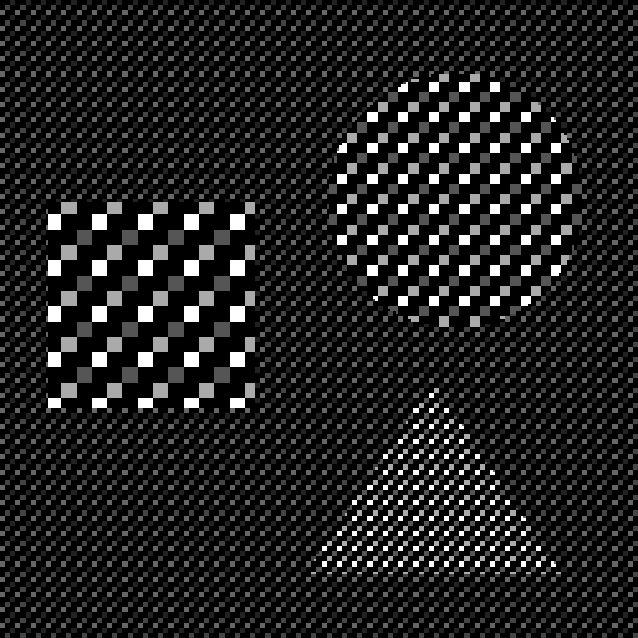}&
        \includegraphics[width=0.22\textwidth]{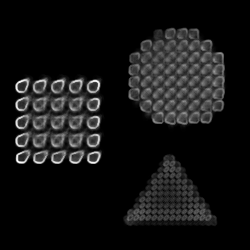}&
        \includegraphics[width=0.22\textwidth]{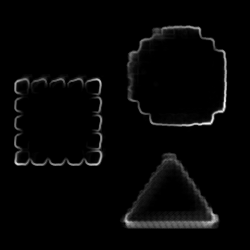}&
        \includegraphics[width=0.22\textwidth]{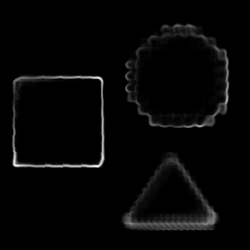}   
    \end{tabular}	
	\caption{(a) Synthetic texture image with different sizes of texton. For different $r$, (b) $r = 4$, (c) $r = 6$, (d) $r = 8$, TEP finds different scale of texture.  (Other parameters are fixed as $\lambda = 0.02, R = 20$.)  When $r$ is small, only texture with small textons are identified as one region.  In (b) the square, each texton is identified as one region.  When $r$ is large, even the large texture is identified as one region, and clearer texture edge is found. }
    \label{fig:scale experiment}
\end{figure}
Figure \ref{fig:starfish} shows real example in (a), using $r=1$ for (b) and $r=7$ for (c).   In (b), the starfish shape is identified but with many details, since with a small $r$, only very small scale texture is identified as one region.   In (c), with a large $r$, larger textures, e.g., inside the starfish, is identified as one textured region, and only boundary of the starfish is emphasized.  As the patch width parameter increase, TEP focuses on large scale structure of the given image, while grouping small details as one region.
\begin{figure}
    \centering
    \begin{tabular}{ccc}
        (a) & (b) & (c)\\
        \includegraphics[width = 0.27\linewidth]{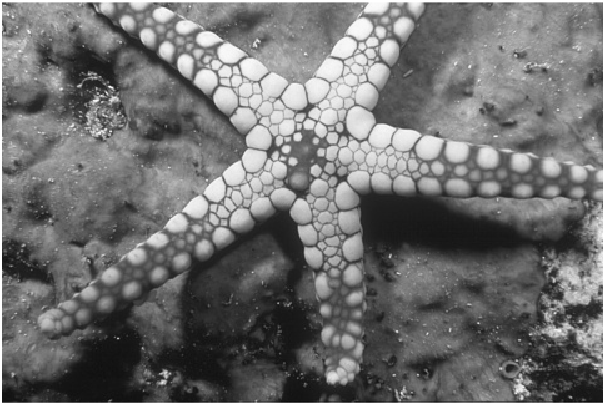}&
        \includegraphics[width = 0.27\linewidth]{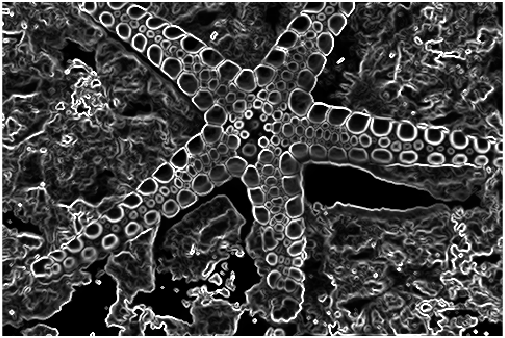}&
        \includegraphics[width = 0.27\linewidth]{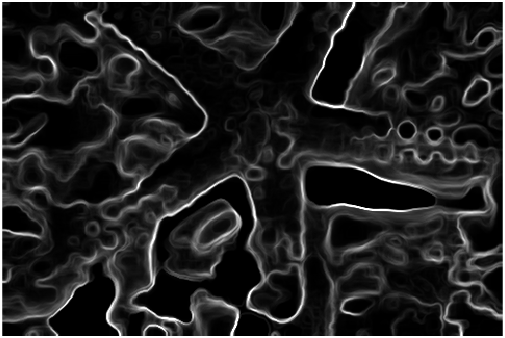}
    \end{tabular}
    \caption{(a) The given image. The TEP results are shown with (b) $r = 1$, (c) $r = 7$. The comparison window parameter $R=35$ is fixed for both results.}
    \label{fig:starfish}
\end{figure}

\subsection{Robustness Against Noise and Multiple junctions}

We test the robustness of TEP against different levels and types of noise. Figure \ref{fig:g_noise} shows the TEP result against additive Gaussian noise with increasing variance, and the TEP result  against increasing salt-and-pepper noise.  In the first row (a), Gaussian noise with variance 0.02, 0.04, 0.06, 0.08 to 0.1 are added from the left to the right.  In the second row (b), they are  TEP results with parameters $r = 5, R = 20$, and $\lambda = 0.018$ (using normalized patch response).   In the third row (c), textured images with Salt and Pepper noise ranging from $10\%, 20\%, 30\%, 40\%$, to $50\%$ are shown from the left to the right.  In the forth row (d), TEP results are shown with the same parameters $r = 5, R = 20$, and $\lambda = 0.018$ (with normalized patch response). 
As more noise are added, some parts of edge strength gets weaker.  Otherwise, TEP shows robustness against Gaussian and Salt and Pepper noise.   In Figure \ref{fig:g_noise} (d) row, clear edge is detected up to 40-50\% of Salt and Pepper noise. 
\begin{figure}
	\centering
    \begin{tabular}{cccccc}	
    (a) & 
\includegraphics[width=.15\textwidth]{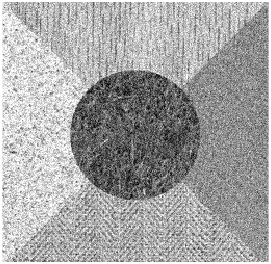}&
  \includegraphics[width=.15\textwidth]{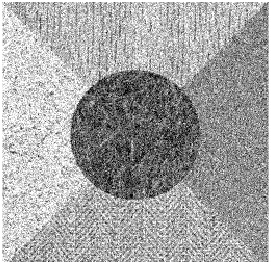}&	
  \includegraphics[width=.15\textwidth]{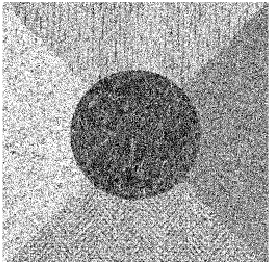}& 	
  \includegraphics[width=.15\textwidth]{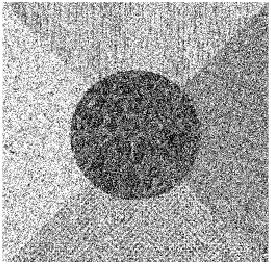}&
\includegraphics[width=.15\textwidth]{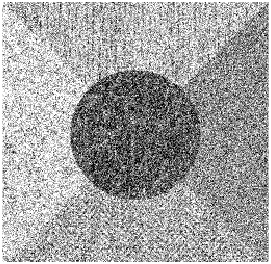}\\		
(b) & 
   \includegraphics[width=.15\textwidth]{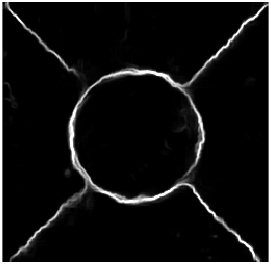}&
  \includegraphics[width=.15\textwidth]{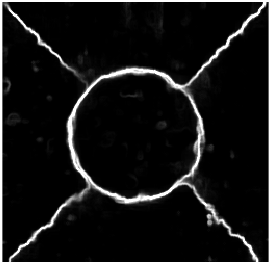}&		
  \includegraphics[width=.15\textwidth]{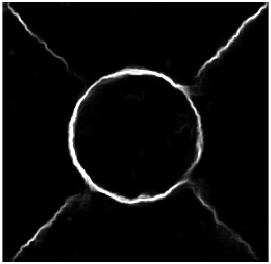}& 	
  \includegraphics[width=.15\textwidth]{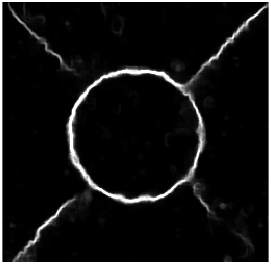}&
  \includegraphics[width=.15\textwidth]{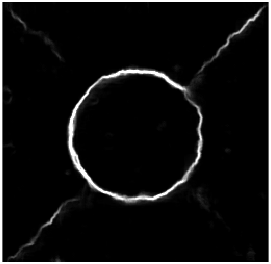}
    \end{tabular}
        \begin{tabular}{cccccc}	
        (c) & 	
\includegraphics[width=.15\textwidth]{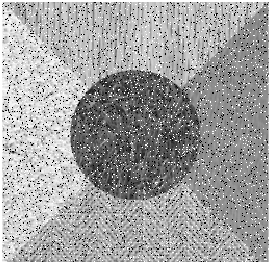}&
  \includegraphics[width=.15\textwidth]{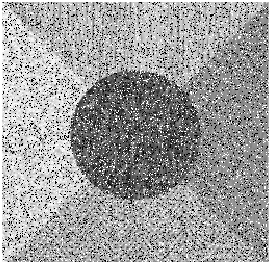}&	
  \includegraphics[width=.15\textwidth]{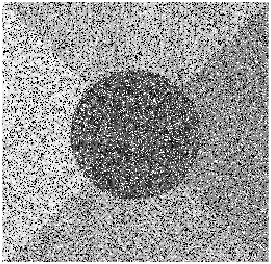}& 	
  \includegraphics[width=.15\textwidth]{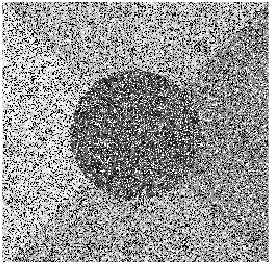}&
 \includegraphics[width=.15\textwidth]{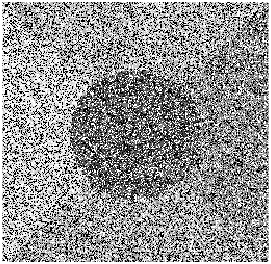}\\	
 (d) & 
   \includegraphics[width=.15\textwidth]{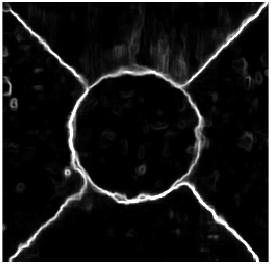}&
\includegraphics[width=.15\textwidth]{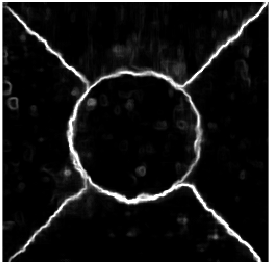}&	\includegraphics[width=.15\textwidth]{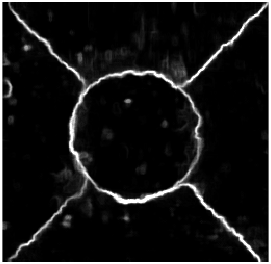}& 
  \includegraphics[width=.15\textwidth]{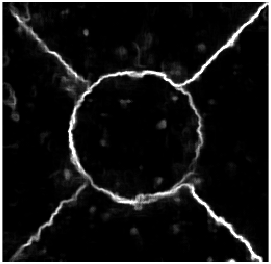}&
  \includegraphics[width=.15\textwidth]{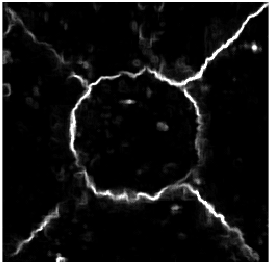}
    \end{tabular}
\caption{First row (a): Textured image with additive Gaussian noise. From the first to the last image, Gaussian noise with variance 0.02, 0.04, 0.06, 0.08 to 0.1 are added (which corresponds to SNR = $17.01, 14.55, 13.30, 12.48$, to $11.89$).  
Second row (b) shows the results of TEP showing the edge detection. 
Third row (c): Textured image with Salt and Pepper noise ranging from $10\%, 20\%, 30\%, 40\%$, to $50\%$.   
Forth row (d) shows the results of TEP showing the edge detection.  Parameters $r = 5, R = 20$, and $\lambda = 0.018$ (with normalized patch response) are used for both set of experiments. TEP shows robustness against noise. }
\label{fig:g_noise}    
\end{figure}

When detecting edges, finding sharp junctions can be challenging, e.g., due to multiple edges meeting at one point, it can easily get blurry.   In Figure \ref{fig:junction gradual change}, we experiment with images with multiple junctions of various textures. 
TEP computes the non-binary edge function $V$ by collecting local segmentation results, that as long as the texture can be separated, multiple junction can also be identified.  This is consistent with equation \eqref{eq:diffOfResponse} in Section \ref{sec:analysis}, the differences of average intensities help the proposed method to identify the texture boundaries.
Figure \ref{fig:junction gradual change} (b) and (d) show TEP results showing clear edges near the junction point.  The experiment results show TEP well handles $N-$junction problem with $N = 4$ and $N = 8$.
In (d), the strength of the edge, $V$ value may be not as high for some points near the junction in the center, and for a very accurate edge detection, multiple junction will impose challenges.  One can further improve the sharpness of edge detection with small modifications, which we further discussed in Appendix \ref{AS:junctions}.  

\begin{figure}
    \centering
    \begin{tabular}{cccc}
        (a)&(b)&(c)&(d)\\
         \includegraphics[width = 0.22\textwidth]{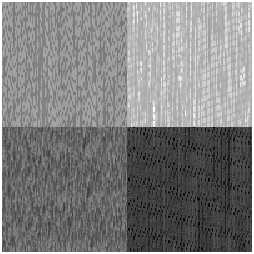}&
         \includegraphics[width = 0.22\textwidth]{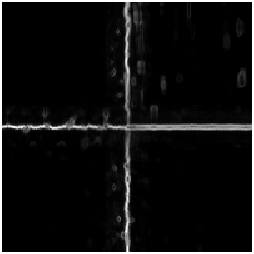}&
        \includegraphics[width = 0.22\textwidth]{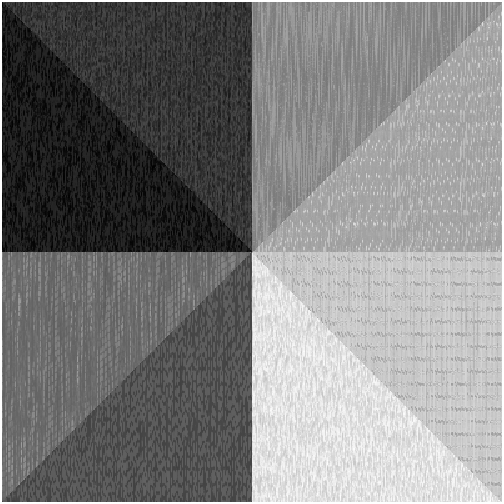}&
         \includegraphics[width = 0.22\textwidth]{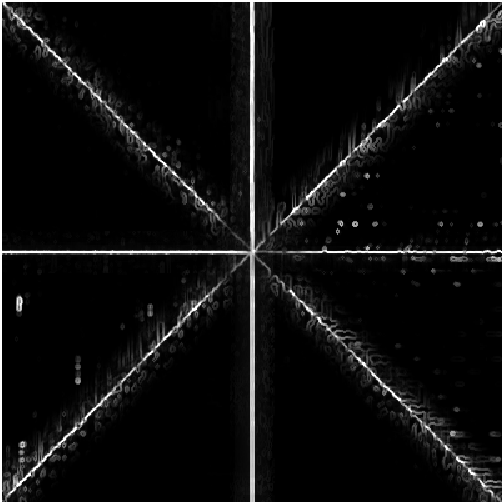}
        \end{tabular}
    \caption{(a) Texture image with 4 junction.  (b) TEP result of (a). (c) Texture image with 8-junction. (d) TEP result of (c).  TEP well handles $N-$junction problem with $N = 4$ and $N = 8$.  }
    \label{fig:junction gradual change}
\end{figure}

\subsection{Image segmentation using the edge function $V$ and image decomposition} \label{sec:segmentation}

Using the edge function $V$, we can design a color segmentation method, using chromaticity and brightness model \cite{chan2001total}.  We separate the given color image $\mathbf{U}_0: \Omega \to \mathbb{R}^3$ to the brightness $U_b: \Omega \to \mathbb{R}$ and the chromaticity $\mathbf{U}_c: \Omega \to \mathbb{S}^3 = \{\mathbf{x}\in\mathbb{R}^3 \mid \|\mathbf{x}\|_2 = 1\}$:
\begin{align*}
    \mathbf{U}_0 = U_b \cdot \mathbf{U}_c, \text{ where } 
    U_b = |\mathbf{U}_0| \text{ and } \mathbf{U}_c = \frac{\mathbf{U}_0}{U_b}.
\end{align*}
We use the edge function, and propose the following \textit{segmentation} functional for each chromaticity and brightness components, with $ \Tilde{\mathbf{U}} = \Tilde{U}_b \cdot \Tilde{\mathbf{U}}_c$,
\begin{align}
    \Tilde{U}_b &= \argmin_{U}  \int_{\Omega}  \left\{ g_\alpha ( V) |\nabla U|^2  + \gamma_1 |U - U_b|^2 \right\} d\mathbf{x}, \label{fnal:bright}\\
    \Tilde{\mathbf{U}}_c &= \argmin_{\mathbf{U}} (\mathbf{U}) = \int_{\Omega}  \left\{ g_\alpha (V) |\nabla \mathbf{U}|^2  + \gamma_2 |\mathbf{U} - \mathbf{U}_c|^2 + \beta (1 - |\mathbf{U}|)^2 \right \} d\mathbf{x},\nonumber 
\end{align}
where $g_\alpha(x):[0, 1] \to [0, 1]$ is an edge indication function,
$\displaystyle{ g_\alpha(x) = \frac{1-x^\alpha}{1 + x^\alpha}}$, 
such that $g_\alpha$ is strictly decreasing, and $g_\alpha(0) = 1,~ g_\alpha(1) = 0$ for $\alpha>0$.  
In order to utilize texture edge $V$ in an effective way, $g_\alpha$ needs to control the smoothness of $U$ inversely proportional to the strength of $V$ within the range of $V \in [0, 1]$.  In application, since $V$ is generated through consensus, $1-V$ is far from zero at texture edge, we choose $\alpha < 1$ to enhance the convexity of $g_\alpha (V)$, which creates wider region near $V = 1$ thus properly control the smoothness of $U$.

The functionals \eqref{fnal:bright} are minimized by considering the Euler-Lagrange equations with gradient decent with time evolution:
\begin{align*} 
    \frac{\partial\Tilde{U}_b}{\partial t} &= \operatorname{div}(g\nabla \Tilde{U}_b) + \gamma_1(\Tilde{U}_b - U_b),\\
    \frac{\partial\Tilde{\mathbf{U}}_c}{\partial t} &= \operatorname{div}(g\nabla \Tilde{\mathbf{U}}_c) + \gamma_2(\Tilde{\mathbf{U}}_c - \mathbf{U}_c) + \beta(1 - \frac{1}{|\Tilde{\mathbf{U}}_c|})\Tilde{\mathbf{U}}_c,
\end{align*}
using finite difference scheme.   We only used the brightness of the image to compute $V(\textbf{x})$.
Figure \ref{fig:demo_segmentation} (a) is the given image, (b) two-phase clustering  of image (a), (c) shows the segmentation result of \eqref{fnal:bright} and (d) is two-phase clustering  of image (c).  Within each region, small scale details are removed, while the edge is kept very sharp. 

\begin{figure}
\centering
 \begin{tabular}{cccc}  
 (a) & (b) & (c) & (d) \\
\includegraphics[width=0.22\textwidth, height=1.5in]{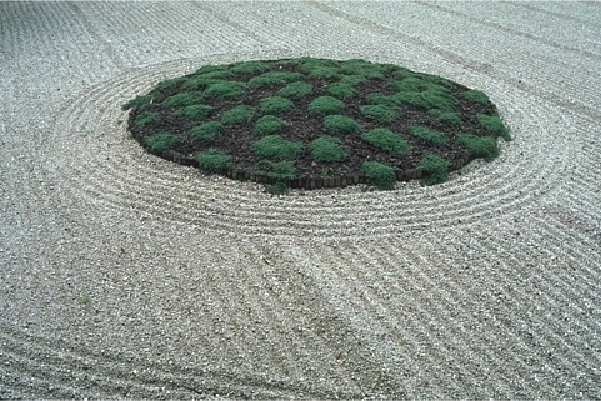}&
\includegraphics[width=0.22\textwidth,height=1.5in]{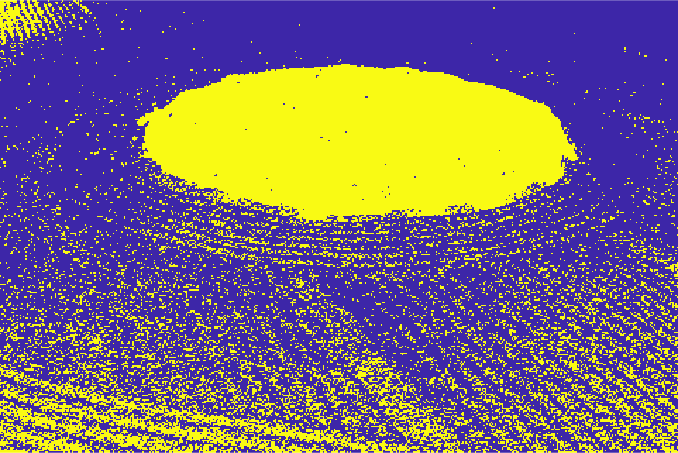} &
\includegraphics[width=0.22\textwidth,height=1.5in]{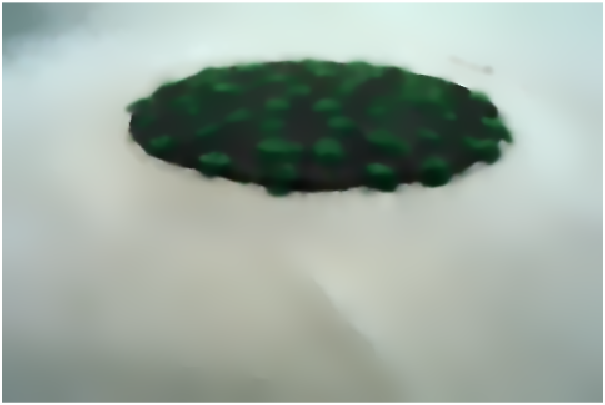} &
\includegraphics[width=0.22\textwidth,height=1.5in]{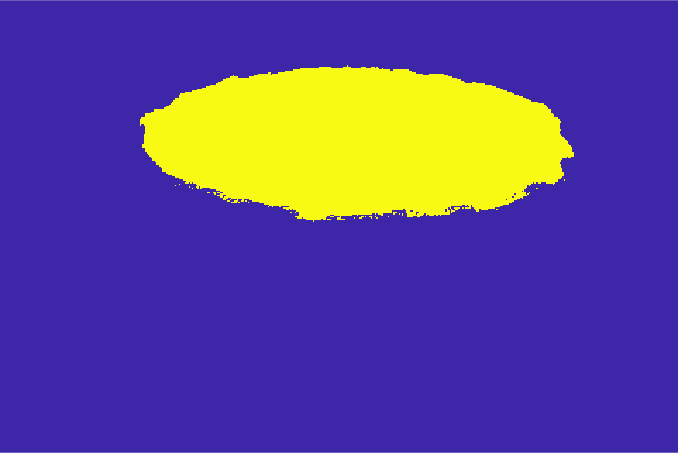}
\end{tabular}
\caption{(a) The given color image.  (b) Two-phase clustering of image (a).   (c) The proposed segmentation result in \eqref{fnal:bright}.  (d) Two-phase clustering of image (c).  While within each region is diffused, the texture edge is kept very sharp, clearly segmenting the image.  Notice how well texture edge is found using $V$. } 
\label{fig:demo_segmentation}
\end{figure}

We present more segmentation results in Figure \ref{fig:moreSeg}.  The brightness of the image is used to compute $V(\textbf{x})$, and we used  parameters  $r = 5, R = 25$, and $\lambda  = 400$, with $U \in [0,255]$. For $g_\alpha$,  $\alpha = 0.2$ is used and  $dt = 0.1$ for evolution.
The first row, the grains are identified as one texture, as well as some textures on the ground as another texture.  In the second row, branches with leafs, and grass regions are identifies as different textures.  In the third row, grains on the rock are identified as one texture, and they are well separated from the fur of the animal, even when the colors are  similar. In the forth row, the texture within the coral are identified as a texture, and detail of the oscillatory boundary are well identified.  
\begin{figure}
    \centering
    \begin{tabular}{ccc}    
(a) & (b) & (c) \\
\includegraphics[width=0.27\textwidth]{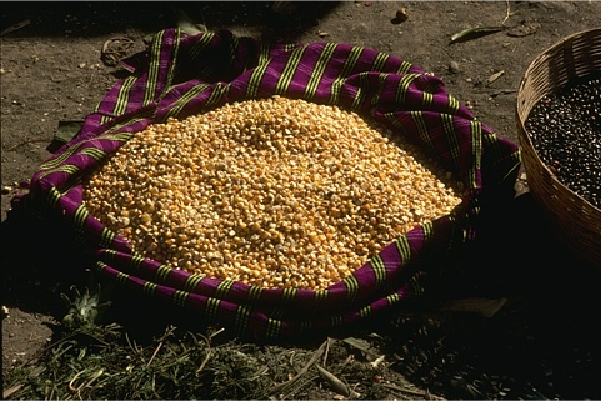}&
\includegraphics[width=0.27\textwidth, frame]{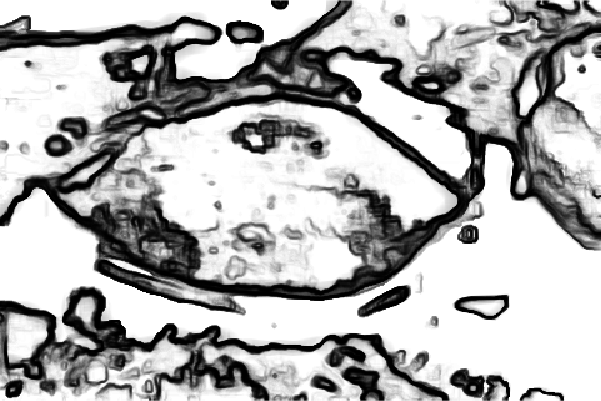}&
\includegraphics[width=0.27\textwidth]{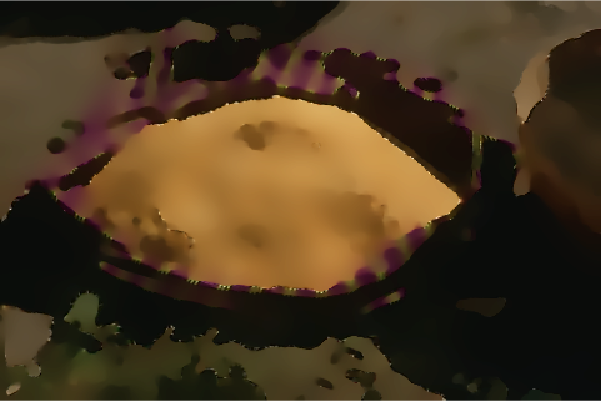}\\
\includegraphics[width=0.27\textwidth]{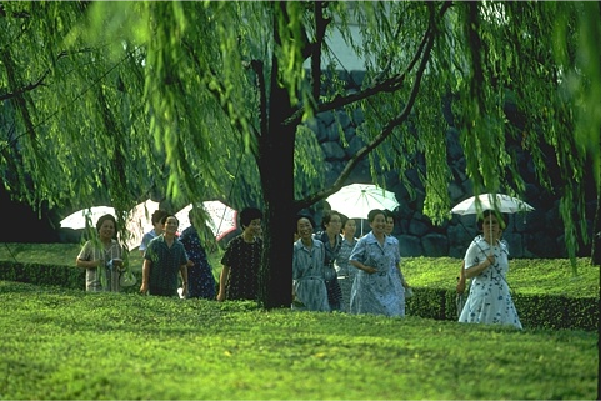}&
\includegraphics[width=0.27\textwidth, frame]{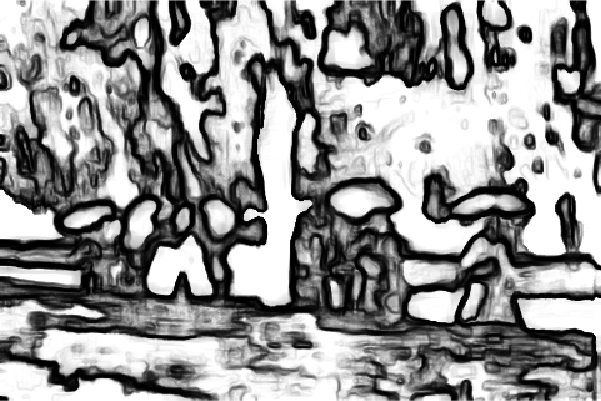}&
\includegraphics[width=0.27\textwidth]{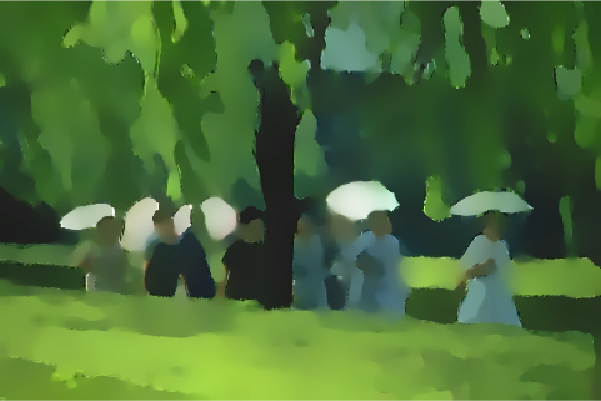}\\
\includegraphics[width=0.27\textwidth]{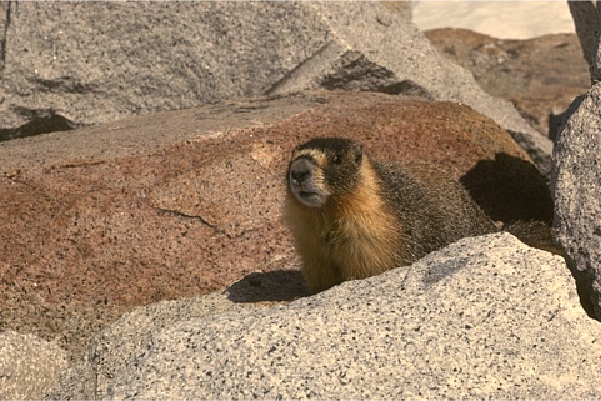}&
\includegraphics[width=0.27\textwidth, frame]{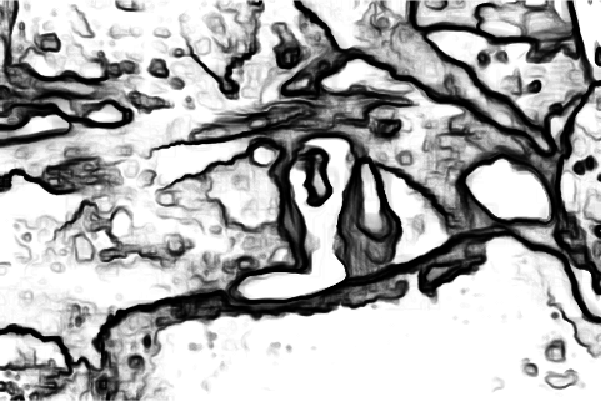}&
\includegraphics[width=0.27\textwidth]{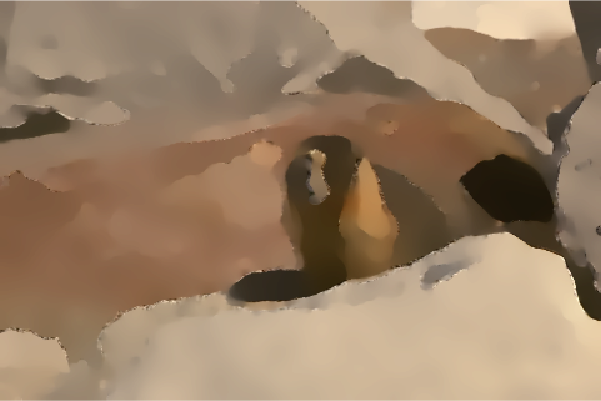}\\
\includegraphics[width=0.27\textwidth]{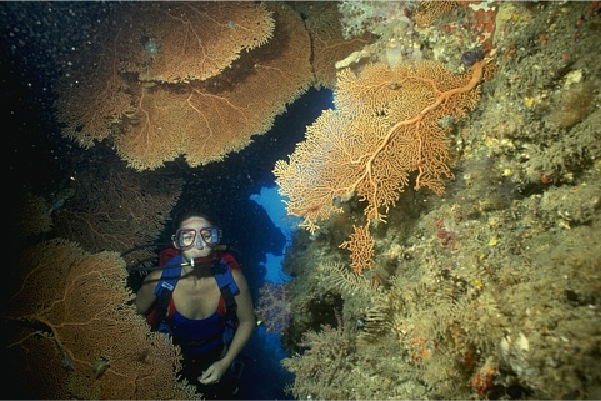}&
\includegraphics[width=0.27\textwidth, frame]{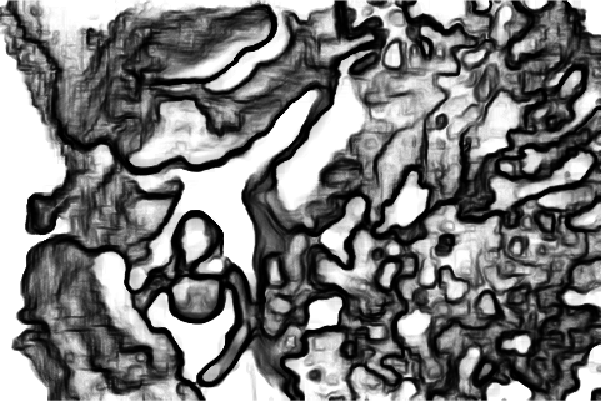}&
\includegraphics[width=0.27\textwidth]{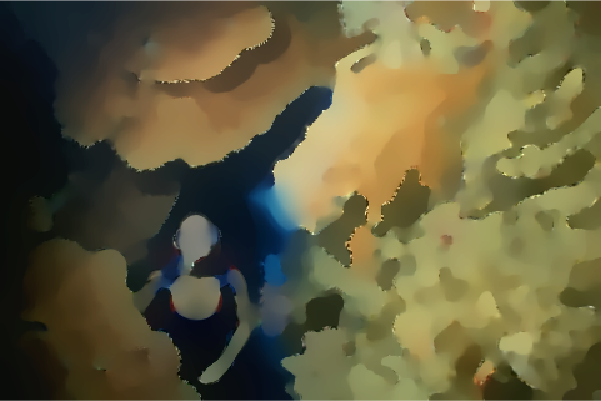}\\
\includegraphics[width=0.27\textwidth]{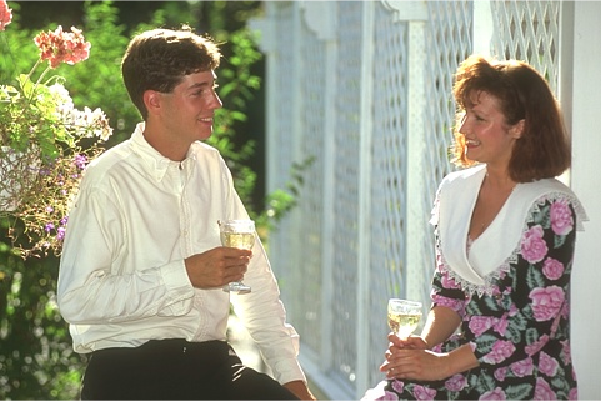}&
\includegraphics[width=0.27\textwidth, frame]{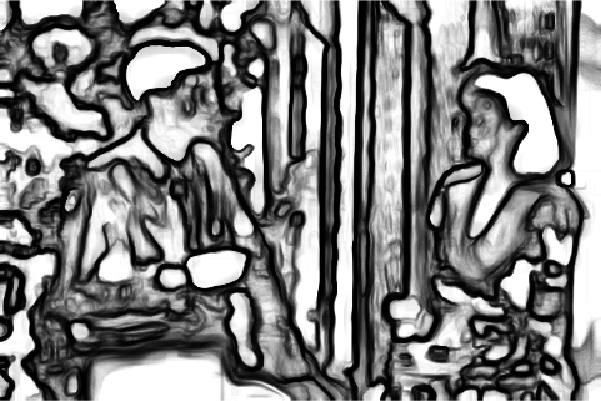}&
\includegraphics[width=0.27\textwidth]{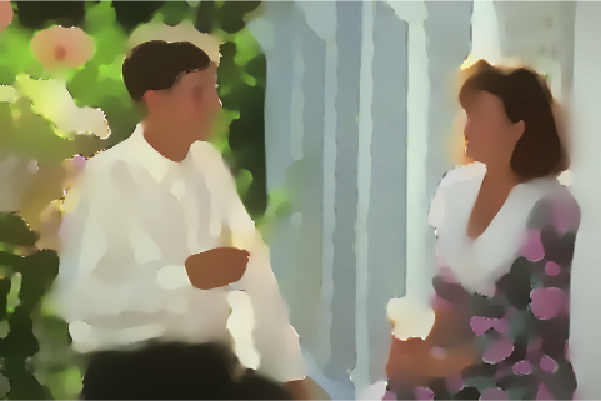}
    \end{tabular}
\caption{(a) The given image. (b) The edge function $g$ using $V$, which ignores the texture within regions.  (c) The segmentation result of \eqref{fnal:bright}. 
Textures of the images are well grouped as each region, e.g., in the top image, grains are identified as one region, and in the third row, grains on the rock are identified as one texture, and they are will separated from the fur of the animal, even when the colors are  similar.  }\label{fig:moreSeg}
\end{figure}
This method can be naturally extended to  image decomposition, and in Figure \ref{fig:decomp}, we present the remainder after the segmentation showing the details of the image. 
\begin{figure}
    \centering
\begin{tabular}{ccc}    
(a) & (b) & (c) \\
\includegraphics[width=0.27\textwidth]{images/86016.png}&
\includegraphics[width=0.27\textwidth]{images/86016_d.png}&
\includegraphics[width=0.27\textwidth]{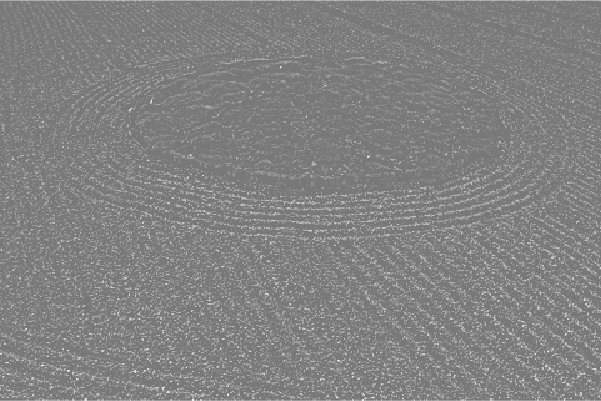}\\
\includegraphics[width=0.27\textwidth]{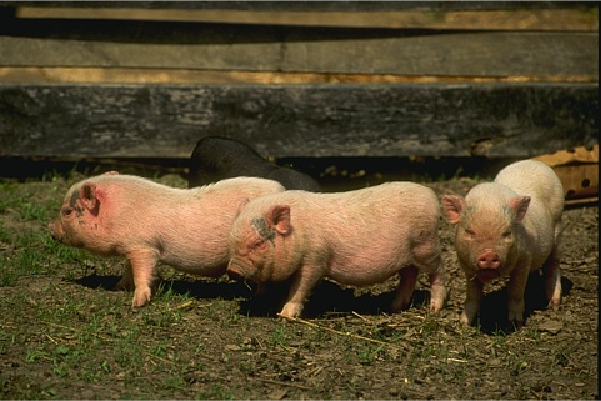}&		\includegraphics[width=0.27\textwidth]{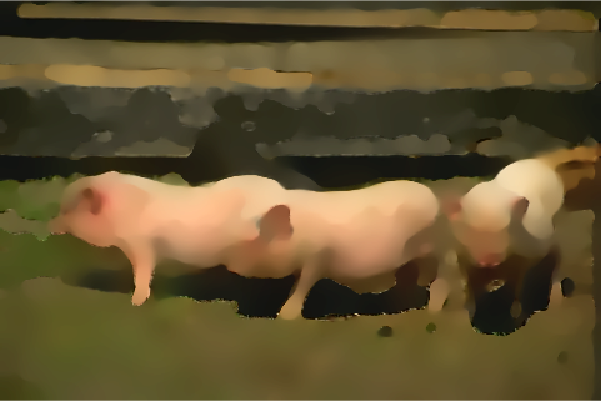}&
\includegraphics[width=0.27\textwidth]{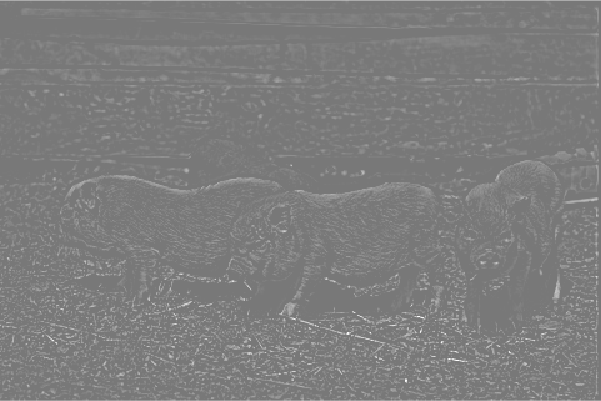} 
\end{tabular}
\caption{(a) The given image. (b) The segmentation result using \eqref{fnal:bright}. There is a natural extension to image decomposition and (c) shows the reminder, subtracting image (b) from (a), representing the details of the image.}\label{fig:decomp}
\end{figure}

In Appendix, we further discuss more details of the proposed method, e.g. proof of Theorem \ref{thm: expectation of R}, behavior of periodic texture, and junction refinement. 

\section{Concluding Remarks} \label{sec:conclusion}
We proposed a texture edge detection method which utilize patch based consensus.  We use the local patch and its response, to emphasize the similarities and the differences among textures, and segmentation of the patch response helps to locate the edge location clearly.  On the boundary of texture, local patch information and patch response is not as accurate to identify edge location, that we utilize the neighbor consensus to stabilize the process.   We statistically analyze when the texture can be separated, and derive necessary conditions to distinguish textures.  
The proposed method has three parameters which are not very sensitive to choose, and show how patch width and the size of texton is related.  This method is robust to different type of noise, and can handle multiple junctions. This method is training-free and filter-free approach.  We provided numerical details and various experiments which illustrate the properties of the proposed model.

In the future, one may consider refining and thinning the edge thickness, which will also improve image decomposition application in Figure \ref{fig:decomp}.  We can also consider using multi-scale approach to further separate more object related edges, e.g., using different scales of patch responses.  
We may improve the performance of TEP via utilizing a scheme with adaptive patch size which can handle more complicated real images. Also, different types of kernels can be used instead of squared euclidean distance when comparing image patches, in order to enhance the sensitivity to certain types of textures.

\bibliographystyle{plain}
\bibliography{references}

\begin{appendices}
\section{Proof of Theorem \ref{thm: expectation of R}. } \label{app:proof of T1}
\begin{proof}
Note that $\mathcal{R}(\mathbf{y};\mathbf{x})$ is a quadratic function of two random vectors $\vec{\mathcal{P}}(\mathbf{x})$ and $\vec{\mathcal{P}}(\mathbf{y})$, the expectation needs to be computed with double integral
\[
    \mathbb{E}\left(\mathcal{R}(\mathbf{y};\mathbf{x})\right) = \mathbb{E}_{\mathbf{x}}\left(\mathbb{E}_{\mathbf{y}|\mathbf{x}}\mathcal{R}(\mathbf{y};\mathbf{x})\right).
\]
To handle the conditional expectation $\mathbb{E}_{\mathbf{y}|\mathbf{x}}(\cdot)$, one needs the conditional distribution of $\vec{\mathcal{P}}(\mathbf{y})$, which is Gaussian with mean \eqref{eq:conditional expectation} and variance \eqref{eq:conditional variance}. Then 
\begin{align}
    \mathbb{E}_{\mathbf{y}|\mathbf{x}}\left(\|\vec{\mathcal{P}}(\mathbf{y}) - \vec{\mathcal{P}}(\mathbf{x})\|_2^2\right) &= \mathbb{E}_{\mathbf{y}|\mathbf{x}}\left(\vec{\mathcal{P}}^T(\mathbf{y})\vec{\mathcal{P}}(\mathbf{y}) - 2\vec{\mathcal{P}}^T(\mathbf{y})\vec{\mathcal{P}}(\mathbf{x}) + \vec{\mathcal{P}}^T(\mathbf{x})\vec{\mathcal{P}}(\mathbf{x})\right).\label{eq:expended quadratic}\\
    &= \mathrm{tr}\left(\Sigma_p(\mathbf{y;x})\right) + \vec{\mu}_p^T(\mathbf{y};\mathbf{x})\vec{\mu}_p(\mathbf{y};\mathbf{x}) -2\vec{\mu}_p^T(\mathbf{y};\mathbf{x})\vec{P}(\mathbf{x}) + \vec{\mathcal{P}}^T(\mathbf{x})\vec{\mathcal{P}}(\mathbf{x})\label{eq:marginalExpectation}
\end{align}
where Lemma \ref{lemma: expectation of random vector} is applied to the first term of the right hand side of \eqref{eq:expended quadratic}. To compute the expectation of \eqref{eq:marginalExpectation} with respect to $\vec{\mathcal{P}}(\mathbf{x})$, we have the following identities:
\begin{align*}
    \mathrm{tr}\left(\Sigma_p(\mathbf{y};\mathbf{x})\right) &= \mathrm{tr}\left(\Sigma_p\right) - \mathrm{tr}\left(\Sigma_p^{-1}\Sigma_{\mathrm{c}}^T(\tau)\Sigma_{\mathrm{c}}(\tau)\right)\\
    \mathbb{E}_\mathbf{x}\left(\vec{\mu}_p^T(\mathbf{y};\mathbf{x})\vec{\mu}_p(\mathbf{y};\mathbf{x})\right) &= \vec{\mu}_p^T\vec{\mu}_p + \mathrm{tr}\left(\Sigma_p^{-1}\Sigma_{\mathrm{c}}^T(\tau)\Sigma_{\mathrm{c}}(\tau)\right)\\
    \mathbb{E}_\mathbf{x}\left(\vec{\mu}_p^T(\mathbf{y};\mathbf{x})\vec{\mathcal{P}}(\mathbf{x}) \right) &= \vec{\mu}_p^T\vec{\mu}_p + \mathrm{tr}\left(\Sigma_{\mathrm{c}}(\tau)\right)\\
    \mathbb{E}_\mathbf{x}\left(\vec{\mathcal{P}}^T(\mathbf{x})\vec{\mathcal{P}}(\mathbf{x})\right) &= \vec{\mu}_p^T\vec{\mu}_p + \mathrm{tr}(\Sigma_p).
\end{align*}
With substitution, the expectation of $\mathcal{R}(\mathbf{y};\mathbf{x})$ is then given as
\begin{align}
    \mathbb{E}_{\mathbf{x}}\left(\mathbb{E}_{\mathbf{y}|\mathbf{x}}\mathcal{R}(\mathbf{y};\mathbf{x})\right) &= \frac{1}{d}\mathbb{E}_{\mathbf{x}}\left(\mathbb{E}_{\mathbf{y}|\mathbf{x}}\left(\|\vec{\mathcal{P}}(\mathbf{y}) - \vec{\mathcal{P}}(\mathbf{x})\|_2^2\right)\right)\notag\\
    &= \frac{1}{d}\left(2\mathrm{tr}\left(\Sigma_p\right) - 2\mathrm{tr}\left(\Sigma_c(\tau)\right) \right)
    = 2\sigma_p^2\left(1 - \exp(-\frac{\tau^2}{2l_p^2})\right)\notag.
\end{align}
\end{proof}

\section{Junction edge refinement}\label{AS:junctions}

For the edges near a junction point, the strength of $V$ may be weaker than straight edges as seem in Figure \ref{fig:junction gradual change}.  This is because the local patches at a  location $\textbf{x}$ which is near a junction point, may  observed another textures which is a bit further from the two immediate two texture edge near $\textbf{x}$ and confuse the segmentation.  In Figure \ref{fig:junction focus problem 2}, (b) shows this effect.  This can be improved by a simple dilation-erosion operation from mathematical morphology \cite{serra1992overview}.  By using line shaped structuring element, where the orientation of the line should be parallel to the existing edge direction, the edge function $V$ is improved as in Figure \ref{fig:junction focus problem 2} (c) and (d). 
\begin{figure}
    \centering
    \begin{tabular}{cccc}
        (a)&(b)&(c)&(d)\\
         \includegraphics[width = 0.22\textwidth, frame]{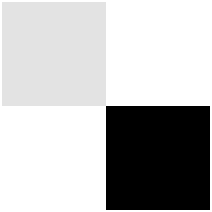}&
         \includegraphics[width = 0.22\textwidth]{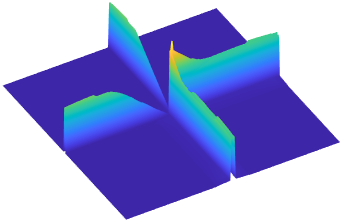}&
         \includegraphics[width = 0.22\textwidth]{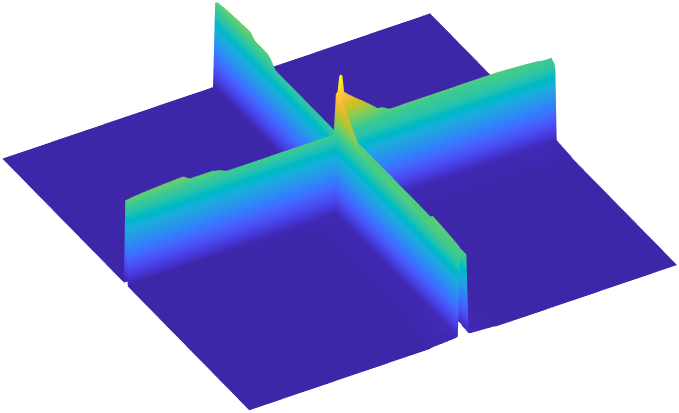} &
        \includegraphics[width = 0.22\textwidth]{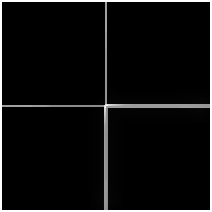}
    \end{tabular}
    \caption{(a) The given image with 4-junction. (b) The edge function $V$.  (c) Improved edge function $V$ using straight line direction enhancing. (d) TEP result with the improved $V$.}
    \label{fig:junction focus problem 2}
\end{figure}

\section{Periodic texture and its patch response}

Extending from the analysis in Section \ref{sec:analysis}, we numerically present the cases for periodic texture.  For periodic texture, it is interesting to notice that the variance of expectation is also strongly correlated to the period of the texture.  In Figure \ref{fig:periodic}, we show the distribution of $\mathbb{E}_{\mathbf{y}|\mathbf{x}}\left(\mathcal{R}(\mathbf{y}; \mathbf{x})\right)$ with varying patch width parameter $r$.  The variance is near zero whenever the patch width parameter $r$ matches the periods of the texture, while the general decreasing effect discussed in section \ref{sec:analysis} still exists.  
This is consistent with the work by Hong, et al.~\cite{hong2008scale}, where the authors observed that for a periodic image, some statistical distance measurement of the image patch vs the entire image vanishes whenever the patch width parameter is a multiple of the texture period.  In another work \cite{jones2009local}, authors measured the scale of the texture by applying time varying Gaussian kernel to the image, and observe when the averaging process has big jump, and use it to measure the scale. 

In this paper,  we choose relatively small $r$ while keeping the patch response consistent and stable.  For periodic texture, we can use these estimations to help find the scale of the texture. 
\begin{figure}
\centering
    \begin{tabular}{cc}
        (a)&(b)  \\
        \includegraphics[width= 0.27\textwidth]{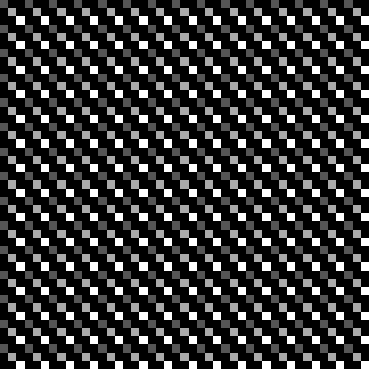}&
        \includegraphics[width= 0.35\textwidth]{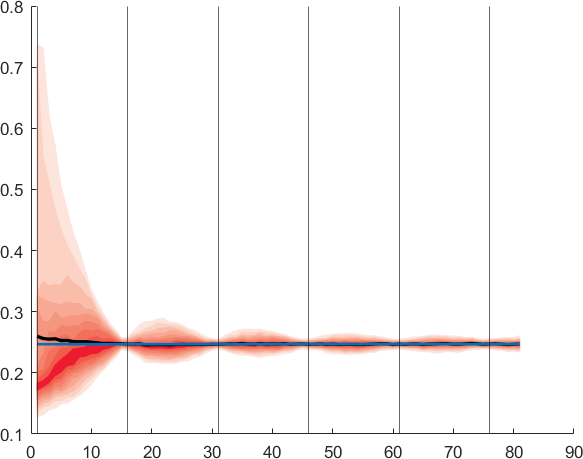}
    \end{tabular}				
\caption{(a) Given synthetic periodic texture. (b) Estimation of the distribution of $\mathbb{E}_{\mathbf{y}|\mathbf{x}}\left(\mathcal{R}(\mathbf{y}; \mathbf{x})\right)$. The black horizontal line in (b) shows empirical mean, and the blue horizontal line in (b) shows the theoretical mean. 
The dark red region shows the location of distribution median. The vertical lines indicate multiples of the texture period. Notice that there is a sharp decrease of the variance when the patch width parameter $r$ hits the period.}
\label{fig:periodic}
\end{figure}

\end{appendices}
\end{document}